\documentclass[journal]{IEEEtran}

\usepackage{tabularx}
\usepackage[ruled,linesnumbered]{algorithm2e}
\usepackage[caption=false,font=footnotesize]{subfig}
\usepackage{fixltx2e}
\usepackage{stfloats}

\usepackage{makecell}

\usepackage{colortbl}  
\usepackage{url}
\usepackage[dvipsnames]{xcolor}
\usepackage{multirow}
\usepackage[colorlinks,bookmarksopen,bookmarksnumbered,citecolor=blue,urlcolor=red]{hyperref}
\usepackage{verbatim}
\usepackage{bbding}
\usepackage{booktabs}
\usepackage{blkarray}

\usepackage{amsthm}
\usepackage{amsfonts}
\usepackage{amssymb}
\usepackage{amsmath}
\usepackage{arydshln} 
\usepackage{mathtools} 
\usepackage{bm}

\usepackage{cite}

\usepackage{pifont}

\usepackage{colortbl} 

\usepackage{extarrows}

\usepackage[english]{babel}
\usepackage{nomencl}
\makenomenclature

\usepackage{etoolbox}
\makeatletter
\patchcmd{\@makecaption}
  {\scshape}
  {}
  {}
  {}
\makeatletter
\patchcmd{\@makecaption}
  {\\}
  {.\ }
  {}
  {}
\makeatother

\usepackage[font=footnotesize,labelfont=bf,labelsep=period]{caption}

\SetKwComment{Comment}{/* }{ */}
\newtheorem{lemma}{\rm \textbf{Lemma}}

\newtheorem{theorem}{\rm \textbf{Theorem}}

\newtheorem{definition}{\rm \textbf{Definition}}
\newtheorem{remark}{\rm \textbf{Remark}}

\renewcommand \arraystretch{1.2}

\DeclareMathOperator{\Img}{Img}
\DeclareMathOperator{\Ker}{Ker}
\DeclareMathOperator{\rank}{rank}
\DeclareMathOperator*{\Span}{span}

\DeclareMathOperator{\col}{col}

\makeatletter
\renewcommand{\maketag@@@}[1]{\hbox{\m@th\normalsize\normalfont#1}}%
\makeatother

\begin{document}


\title{A Transformation-based Consistent Estimation Framework: Analysis, Design and Applications}

\author{Ning Hao, Chungeng Tian, and Fenghua He
\thanks{Chungeng Tian is Ph.D candidate with Harbin Institute of Technology, Harbin, 150000, China. {\tt (email: 2018tian@gmail.com)}} 
\thanks{Ning Hao and Fenghua He are faculties with Harbin Institute of Technology, Harbin, 150000, China. {\tt (email: hefenghua@hit.edu.cn; haoning0082022@163.com)}}
}


\maketitle

\begin{abstract}
In this paper, we investigate the inconsistency problem arising from observability mismatch that frequently occurs in nonlinear systems such as multi-robot cooperative localization and simultaneous localization and mapping. For a general nonlinear system, we discover and theoretically prove that the unobservable subspace of the EKF estimator system is independent of the state and belongs to the unobservable subspace of the original system. On this basis, we establish the necessary and sufficient conditions for achieving observability matching. These theoretical findings motivate us to introduce a linear time-varying transformation to achieve a transformed system possessing a state-independent unobservable subspace. We prove the existence of such transformations and propose two design methodologies for constructing them. Moreover, we propose two equivalent consistent transformation-based EKF estimators, referred to as T-EKF 1 and T-EKF 2, respectively. T-EKF 1 employs the transformed system for consistent estimation, whereas T-EKF 2 leverages the original system but ensures consistency through state and covariance corrections from transformations. To validate our proposed methods, we conduct experiments on several representative examples, including multi-robot cooperative localization, multi-source target tracking, and 3D visual-inertial odometry, demonstrating that our approach achieves state-of-the-art performance in terms of accuracy, consistency, computational efficiency, and practical realizations.
\end{abstract}


\section{Introduction}
For nonlinear estimation problems, such as multi-robot cooperative localization (CL)~\cite{B7, B8}, simultaneous localization and mapping (SLAM)~\cite{B31, B75}, and multi-source target tracking~\cite{B67}, a significant challenge is the inconsistency issue. As defined in \cite{B2}, a state estimator is \emph{consistent} if the estimation errors are zero-mean and the covariance matches that computed by the estimator. When this condition is not met, the covariance calculated by the estimator tends to be overly optimistic, leading to erroneous information acquisition and severe performance degradation \cite{B63, B56, B76}. The study by \cite{B37} analyzed the inconsistency issue from the perspective of system observability, and revealed that the root cause of inconsistency lies in the mismatch of observability properties between the estimator system and the underlying nonlinear system. Specifically, for both CL and SLAM, it was shown that the estimator system has fewer unobservable directions than the actual system. Consequently, the estimator mistakenly gathers information along the unobservable directions, ultimately leading to inconsistency.

To address the inconsistency issue caused by the mismatch of observability properties, many approaches have been proposed, which can be roughly classified into four categories: observability constraint-based filters \cite{B12, B14, B6, B13}, robot-centric filters \cite{B53, B54, B55, B59, B29}, matrix Lie group-based filters \cite{B32, B33, B45, B50, B51, B78} and transformation-based filters \cite{B43, B70, B74, B69}.  Observability constraint-based filters mitigate the inconsistency issue by explicitly modifying the estimator system to enforce correct observability properties. Amongst, the first-estimates Jacobian method~\cite{B12, B14} guarantees correct observability properties by adjusting the linearization points, while the observability-constrained approach \cite{B6, B13} directly incorporates the pre-specified unobservable subspace to alter the Jacobians. Although these methods are straightforward to implement and demonstrate impressive performance in improving consistency, the modification might result in deviations from the original first-order approximation, rendering the estimator system not theoretically optimal anymore.

The robot-centric and matrix Lie group-based estimators endeavor to adopt novel state representations that enable the estimator system to inherently maintain correct observability properties. The robot-centric estimators \cite{B53, B54, B55, B59, B29} adopt a robot-centric coordinate system to formulate the state, while the matrix Lie group-based filters \cite{B32, B33, B45, B50, B51, B78} exploit special matrix Lie groups to parameterize the state. Under these novel formulations, correct observability properties can be automatically preserved without any Jacobian modification. Consequently, these approaches achieve better performance in terms of accuracy and consistency. However, one potential drawback of these methods is their lack of a unified and canonical design. Consequently, they may rely on empirical design choices for specific problems, which can be challenging and often less systematic.

The transformation-based estimators leverage transformations to achieve observability-preserved systems \cite{B43, B70, B74, B69}. To the best of our knowledge, \cite{B43} was the first to employ transformations to tackle the inconsistency problem wherein a Kalman decomposition transformation was introduced to eliminate the factors causing observability mismatch. To address inconsistencies in SLAM, \cite{B70} further presented a transformation that facilitates efficient covariance propagation while yielding results comparable to the RI-EKF. Meanwhile, \cite{B74} leveraged transformations to establish connections between existing estimators, offering a unified perspective on state-of-the-art methods. Furthermore, \cite{B69} proposed a general design principle of transformations based on the structure of the unobservable subspace. These estimators effectively exploit existing formulations and are straightforward to implement, achieving excellent performance. However, the design of transformations lacks theoretical guarantees and still requires case-by-case treatments, which limits its extension to broader problems and hinders the exploration of new design methods.

In this paper, we propose a theoretically grounded and systematic framework to tackle the inconsistency problem caused by observability mismatch for general systems. Our work begins with a rigorous theoretical analysis and proof of the observability properties of the EKF estimator system and its relationship to the underlying nonlinear system. Through this analysis, we reveal a pivotal insight: the unobservable subspace of the EKF estimator system is state-independent and aligns with the unobservable subspace of the original nonlinear system if and only if it is state-independent. This finding serves as the foundation of our proposed solution. Building on this insight, we introduce \emph{linear time-varying transformations} to mitigate the inconsistency problem. These transformations allow us to achieve transformed systems with a state-independent unobservable subspace, thereby automatically maintaining correct observability properties. We demonstrate the existence of such transformations and present two design methods. One is to straightforwardly construct them based on the unobservable subspace while the other achieves this by converting the state propagation Jacobians into constant forms. Both guarantee the unobservable subspace of the transformed system becomes state-independent. Within this framework, we develop two equivalent and consistent transformation-based EKF estimators. The first utilizes the transformed system to achieve consistent estimation, while the second operates on the original system but ensures consistency through state and covariance corrections from the transformations. Additionally, we present two state update approaches for both estimators, providing flexibility in their implementation. 

In summary, the main contributions of this paper are as follows:
\begin{itemize}
\item For a general nonlinear system, we discover and theoretically prove that the unobservable subspace of the EKF estimator system is state-independent and belongs to that of the actual system. Furthermore, we establish a necessary and sufficient condition for ensuring observability matching.

\item We propose two design methods of linear time-varying transformations to make the transformed system's unobservable subspace become state-independent. One presents analytical construction based on the unobservable subspace and the other achieves this by converting the state propagation Jacobians into a constant form.

\item We propose two equivalent and consistent transformation-based EKF estimators and present two update approaches for both estimators. These implementations exhibit distinct advantages, allowing us to exploit the advantage of customization for specific purposes.

\item We validate the proposed method in three illustrative applications, demonstrating its advantages in accuracy, consistency, computational efficiency, and practical realizations.
\end{itemize}

The rest of this paper is organized as follows: related works are summarized in Section \ref{sec:related_work}. Section \ref{sec:analysis} analyzes the EKF estimator system's observability and its relationship to the underlying nonlinear system. Section \ref{sec:transformation} details the transformation design method and the transformation-based estimators. Applications of the proposed method in some representative tasks are presented in Section \ref{sec:application_cl}, \ref{sec:application_tt}, and \ref{sec:application_vins}, respectively. Finally, conclusions and future work are summarized in Section \ref{sec:conclusion}.

\section{Related Work}
\label{sec:related_work}

\subsection{Inconsistency Analysis}

Inconsistency refers to the phenomenon wherein an estimator underestimates the uncertainty of state estimates and yields overconfident estimation results. A critical inconsistency issue arises from the mismatch of observability properties between the estimator system and the underlying nonlinear system. This problem was first recognized by \cite{B28} in the context of 2D SLAM. Subsequently, \cite{B15} examined the symptoms of inconsistency and demonstrated that the uncertainty estimation in the orientation direction is the fundamental cause of this inconsistency. Moreover, \cite{B24} analyzed a scenario in which a robot observes a landmark from two positions, revealing that inconsistency may arise due to Jacobians being evaluated at different estimation values.

Based on these, \cite{B16} investigated the inconsistency problem from the perspective of system observability and identified the connection between observability and consistency. For both CL and SLAM, it was proved that the EKF estimator system has fewer unobservable directions than the underlying nonlinear system \cite{B6, B13, B14}. In view of this, it was conjectured that the dimension reduction in the unobservable subspace causes the estimator to surreptitiously gain spurious information along the unobservable direction, and thus be inconsistent. Despite significant progress, the underlying mechanism and potential conditions leading to observability mismatch remain unclear. As a result, most of the existing methods are empirical and rely on a case-by-case treatment. This limits the development of consistent estimators for general nonlinear systems.


\subsection{Observability Constraint-Based Estimators}

\def\txtq#1{{\color{red} #1}}

The observability constraint-based methods mainly consist of first-estimates Jacobian (FEJ) \cite{B12, B7, B8, B14, B18, B34} and observability-constrained (OC) \cite{B6, B13, B41, B42, B58} methodologies. The FEJ chooses the first state estimates as the linearization points to ensure proper observability properties. It is commonly integrated into the framework of Extended Kalman Filter (EKF) and Multi-State Constraint Kalman Filter (MSCKF) \cite{B34}, and has been successfully applied in CL \cite{B7, B8} and SLAM \cite{B12, B14}, achieving significant improvement in consistency. Since the FEJ always evaluates the Jacobians at the initial estimates, poor initial estimates might lead to non-negligible linearization errors, thereby degrading its performance. To address this issue, an extended version of the FEJ was proposed in \cite{B18} to alleviate the linearization errors caused by inadequate initial estimates. Despite effectively eliminating the linearization errors in the measurement Jacobians, additional linearization errors still exist in the state propagation Jacobians.

The OC method maintains the pre-specified unobservable subspace by directly modifying the Jacobians \cite{B6, B13, B41, B42, B58}. Compared to the FEJ, this approach effectively reduces linearization errors and has also been widely applied. \cite{B6, B13} identified the unobservable subspace of visual-inertial navigation systems and successfully applied the OC estimators to address the inherent inconsistency issues. \cite{B41} presented an efficient and consistent design of stereo visual-inertial odometry by embedding OC within the MSCKF framework. \cite{B58} extended it by imposing both state-transition and observability constraints in computing EKF Jacobians, achieving improved performance. Nonetheless, the adopted Jacobians do not strictly adhere to the first-order Taylor expansion and thus are not theoretically optimal.

\subsection{Robot-centric Estimators}

Robot-centric estimators \cite{B53, B54, B55, B59, B29} mitigate the inconsistency issue by reformulating the state with respect to a local moving frame attached to the robot. This approach automatically preserves the system's unobservable subspace, thereby circumventing potential inconsistency challenges. Robot-centric formulations were earliest used to improve consistency in 2D SLAM \cite{B55, B59}, where the world reference frame was included as a non-observable feature in the state vector. \cite{B29, B30} adopted a similar robot-centric formulation as described in \cite{B55, B59} within a sliding-window filtering-based 3D VIO framework. This formulation demonstrated correct observability properties and thus achieved improved consistency. \cite{B53, B54} adopted a fully robot-centric state representation within the iterated EKF visual-inertial odometry framework, which decouples the unobservable states from the other state variables, significantly enhancing consistency. Compared to observability constraint-based methods, robot-centric approaches mitigate the inconsistency issue caused by observability mismatch without necessitating modifications to the Jacobians, thereby yielding superior performance. However, the selection of state representation in these methods does not appear to follow a well-defined design process. This ambiguity may present challenges when attempting to extend robot-centric approaches to address other problems.




\subsection{Matrix Lie Group-Based Estimators}
The conventional estimators treat the state space as an Euclidean space. For many systems, the state actually evolves on a manifold with inherent geometry constraints. In recent years, leveraging the manifold representation has gained popularity in the field of robotics and achieved excellent performance \cite{B47, B46}. The advancements have shown that exploiting the geometric structure of a system can lead to improved convergence and consistency~\cite{B48, B57}. Two typical matrix Lie group-based approaches, i.e., the invariant EKF (I-EKF) \cite{B44, B45, B33, B32} and the equivariant filter (EqF) \cite{B49, B51}, have been extensively applied into SLAM.

The invariant EKF associates uncertainty with an invariant error-state on the matrix Lie group, which does not change under any stochastic unobservable transformation \cite{B31,B44}. This guarantees the uncertainty of the unobservable states does not affect subsequent estimates and circumvents the inconsistency issue caused by observability mismatch. Barrau and Bonnabel et al. \cite{B45, B47} proposed extended special Euclidean groups, i.e., $\rm \mathbf{SE}_{k}(2)$ and $\rm \mathbf{SE}_{k}(3)$, and successfully applied these two group structures into 2D/3D EKF-SLAM. \cite{B31} proved the convergence and improved consistency of $\rm \mathbf{SE}_{k}(3)$ based 3D EKF-SLAM. Using the same matrix Lie group structure, \cite{B32} integrated the invariant error-state into the filtering framework of MSCKF and obtained a consistent state estimator. \cite{B33} combined the invariant state representation with FEJ and showed that it outperforms the classical FEJ-EKF.

The equivariant estimators, which are rooted in the theory of equivariant systems, are aimed to exploit the Lie group symmetry of a system to achieve well-behaved performance. With regard to the problem of visual SLAM (VSLAM), \cite{B51} proposed a novel Lie group structure $\rm \mathbf{VSLAM}_n(3)$ on which the system output is equivariant. \cite{B49} proposed the equivariant filter (EqF), which utilizes the designed Lie group $\rm \mathbf{VSLAM}_n(3)$ and leverages the equivariance of the measurement function to reduce linearization errors. Based on these works, \cite{B50} incorporated the equivariant filter to VIO with a symmetry that is compatible with visual measurements, and achieved an advanced performance in terms of efficiency and accuracy. Meanwhile, \cite{B77} integrated the proposed Lie group structure into the framework of MSCKF, with consistency inherently guaranteed even during convergence phases. By leveraging some special Lie groups for state representation, correct observability properties can be maintained automatically, which ensures consistency and even better convergence. However, the design of matrix Lie groups is rather challenging. It seems to rely on an empirical design and is inconvenient to extend to general nonlinear systems. 



\subsection{Transformation-Based Estimators}

In contrast to robot-centric and matrix Lie group-based estimators that address the inconsistency problem by seeking special state representations, transformation-based filters utilize linear time-varying transformations to mitigate this issue \cite{B43, B70, B74, B69}. \cite{B43} designed and implemented a Kalman decomposition transformation within the estimator system, effectively isolating and eliminating the factors causing observability mismatch. To resolve the inconsistency issue in SLAM, both \cite{B70} and \cite{B74} proposed viable transformations by which the transformed system automatically circumvents the observability mismatch issue. In addition, they demonstrated that leveraging these transformations for filtering can avoid the computational bottleneck associated with the RI-EKF, resulting in significantly enhanced computational efficiency. By examining the unobservable subspace of the estimator system, \cite{B69} pointed out that the observability mismatch issue can be circumvented as long as the actual system's unobservable subspace is state-dependent. Based on this observation, they presented a general design principle for these transformations. Nevertheless, the design method lacks rigorous theoretical analysis and proof, which limits its extension to broader problems.



\section{Observability Mismatch Issue Analysis}
\label{sec:analysis}


In this section, we present a theoretical analysis of the unobservable subspace of the estimator system and its relationship to the underlying nonlinear system. Additionally, we establish the necessary and sufficient conditions for guaranteeing observability matching.

\subsection{Observability Mismatch Issue Statement}

We begin by briefly reviewing the inconsistency issues stemming from observability mismatch, which are commonly found in applications such as multi-robot cooperative localization (CL), target tracking, and simultaneous localization and mapping (SLAM). Without loss of generality, consider the following nonlinear system
\begin{align} \label{equ:dyanmics}
\mathbf{x}_{k+1} 
&= \mathbf{f}(\mathbf{x}_{k}, \: \mathbf{u}_{k}, \: \mathbf{v}_{k}) \\
\mathbf{y}_{k} &= \mathbf{h}(\mathbf{x}_{k}, \: \mathbf{w}_{k}) 
\label{equ:measure}
\end{align}
where $k=0,1,2,...$ represents the discrete-time index; $\mathbf{x}_k \in \mathbb{R}^{n}$, $\mathbf{u}_k \in \mathbb{R}^{q}$ and $\mathbf{y}_k \in \mathbb{R}^{p}$ denote the system's state, input and measurement at time step $k$, respectively; $\mathbf f(\cdot)$ denotes the function encoding the evolution of the system; $\mathbf h(\cdot)$ denotes the observation function; $\mathbf{v}_k$ and $\mathbf{w}_k$ represent the system and measurement noises, respectively, which are zero-mean white Gaussian, i.e., $\mathbf{v}_{k} \sim \mathcal{N}(\mathbf{0}, \mathbf{Q}_{k})$ and $\mathbf{w}_{k} \sim \mathcal{N}(\mathbf{0}, \mathbf{R}_{k})$.

In order to achieve a nice estimate of $\mathbf{x}_{k}$ with dynamics in (\ref{equ:dyanmics}) based on the available noisy measurements $\mathbf{y}_{k}$ given in (\ref{equ:measure}), we generally employ the Extended Kalman Filter (EKF) due to its efficiency and competitive accuracy. The classical EKF provides a state estimate $\hat{\mathbf{x}}_{k|k}$ and a covariance estimate $\hat{\mathbf{P}}_{k|k}$ at the same time. Unluckily, for estimation problems such as CL, SLAM, etc, the classical EKF suffers from \emph{the inconsistency issue}, i.e., the calculated covariance is smaller than the actual state estimation covariance
\begin{equation}
\hat{\mathbf{P}}_{k|k} < \mathbb{E}((\mathbf{x}_{k} - \hat{\mathbf{x}}_{k|k})(\mathbf{x}_{k} - \hat{\mathbf{x}}_{k|k})^\top) . 
\end{equation}
The wrong covariance estimate cannot reflect the actual state uncertainty and leads to wrong information gain, thereby downgrading the estimator's performance.

As discussed in \cite{B37}, the inconsistency problem mainly arises from the mismatch in observability properties between the estimator system and the underlying nonlinear system. Up to now, the relationship of observability properties between the estimator system and the underlying nonlinear system has yet to be rigorously explored. Subsequently, we will present a theoretical analysis of this relationship, guiding the design of consistent estimation methods.


Towards this end, we first introduce the \emph{nominal linearized system}. Let $\{\mathbf{x}_{k}^{\ast}\}_{k \geq 0}$ denote a nominal trajectory of dynamics \eqref{equ:dyanmics} with noises turned off. The current nominal state $\mathbf{x}^{\ast}_{k} \in \mathbb{R}^{n}$ and its previous nominal state $\mathbf{x}^{\ast}_{k-1} \in \mathbb{R}^{n}$ obey the following propagation equation
\begin{equation} \label{equ:linear_point}
\mathbf{x}^{\ast}_{k} = \mathbf{f}(\mathbf{x}^{\ast}_{k-1}, \: \mathbf{u}_k, \: \mathbf{0}).
\end{equation}
The first-order Taylor expansion of \eqref{equ:dyanmics} and \eqref{equ:measure} at each time $k$ around the nominal linearization point $\mathbf{x}^{\ast}_{k-1}$ and $\mathbf{x}^{\ast}_{k}$ is 
\begin{align} \label{equ:linear_state}
\mathbf{e}_{k} &= \mathbf{F}_{k-1} \ \mathbf{e}_{k-1} + \mathbf{G}_{k-1} \ \mathbf{v}_k , \\
\tilde{\mathbf{y}}_{k} &= \mathbf{H}_{k} \ \mathbf{e}_{k} + \mathbf{D}_{k} \ \mathbf{w}_k,
\label{equ:linear_measure}
\end{align}
where 
$
\mathbf{e}_{k} \triangleq \mathbf{x}_k - \mathbf{x}^{\ast}_k$, $\tilde{\mathbf{y}}_{k}  \triangleq \mathbf{y}_{k} - \mathbf{h}(\mathbf{x}^{\ast}_{k}, \: \mathbf{0})$, and  
\begin{align} \label{equ:trans_f}
{\mathbf{F}}_{k-1} &= {\mathbf{F}}(\mathbf{x}_{k-1}^{\ast}), \\
{\mathbf{G}}_{k-1} &= {\mathbf{G}}(\mathbf{x}_{k-1}^{\ast}) , \\
\label{equ:trans_h}
{\mathbf{H}}_{k} &= {\mathbf{H}}(\mathbf{x}_{k}^{\ast}) , \\
{\mathbf{D}}_{k} &= {\mathbf{D}}(\mathbf{x}_{k}^{\ast}) , 
\end{align}
with ${\mathbf{F}}(\cdot)$, ${\mathbf{G}}(\cdot)$, ${\mathbf{H}}(\cdot)$ and ${\mathbf{D}}(\cdot)$ the Jacobian matrix-valued functions \cite{B79} of the nonlinear system \eqref{equ:dyanmics}-\eqref{equ:measure} as follows
\begin{align}
\label{equ:F}
{\mathbf{F}}(\mathbf{x}_{k-1}) &= \frac{\partial \mathbf{f}(\mathbf{x}_{k-1}, \: \mathbf{u}_{k}, \: \mathbf{0})}{ \partial \mathbf{x}_{k-1}}, \\
\label{equ:G}
{\mathbf{G}}(\mathbf{x}_{k-1}) &= \frac{\partial \mathbf{f}(\mathbf{x}_{k-1}, \: \mathbf{u}_{k}, \: \mathbf{0})}{ \partial \mathbf{n}_k}  , \\
\label{equ:H}
{\mathbf{H}}(\mathbf{x}_{k}) &= \frac{\partial \mathbf{h}(\mathbf{x}_{k}, \: \mathbf{0}) }{ \partial \mathbf{x}_k}, \\
\label{equ:D}
{\mathbf{D}}(\mathbf{x}_{k}) &= \frac{\partial \mathbf{h}(\mathbf{x}_{k}, \: \mathbf{0}) }{ \partial \mathbf{w}_k} .
\end{align}


We now recall the definition of observability. For the nominal linearized system \eqref{equ:linear_state}-\eqref{equ:linear_measure}, the local observability matrix~\cite{B1} is defined over a time interval $[k, k+\ell]$ as follows
\begin{equation}
\boldsymbol{\mathcal{O}}_k
\triangleq
\left[
\begin{array}{c}
\mathbf{H}_{k} \\
\mathbf{H}_{k+1} \mathbf{F}_{k} \\
\vdots \\
\mathbf{H}_{k+\ell} \mathbf{F}_{k+\ell-1} \cdots \mathbf{F}_{k} \\
\end{array}
\right]
\label{equ:obs_mat}
\end{equation}
with $\ell \to \infty$. The rank of the observability matrix $\boldsymbol{\mathcal{O}}_k$ is
$$
\rank(\boldsymbol{\mathcal{O}}_k) = n-r \leq n
$$ 
with $n$ the system state's dimension and $r$ the dimension of the observability matrix's kernel space. The nominal linearized system \eqref{equ:linear_state}-\eqref{equ:linear_measure} is fully observable if and only if $\boldsymbol{\mathcal{O}}_k$ is with full column rank, i.e., $r = 0$. Otherwise, the system is partially observable with $r > 0$. The observable subspace of the nominal linearized system \eqref{equ:linear_state}-\eqref{equ:linear_measure} at time step $k$ is exactly the image space of the observability matrix $\boldsymbol{\mathcal{O}}_k$, denoted as $\Img(\boldsymbol{\mathcal{O}}_k)$, whereas the unobservable subspace corresponds to the kernel space of the observability matrix $\boldsymbol{\mathcal{O}}_k$, represented as $\Ker(\boldsymbol{\mathcal{O}}_k)$. The nominal linearized system \eqref{equ:linear_state}-\eqref{equ:linear_measure} has the same observability properties as the underlying nonlinear system \eqref{equ:dyanmics}-\eqref{equ:measure}.


Next, we introduce the observability mismatch issues encountered by the classical EKF. Recalling the nominal linearized system \eqref{equ:linear_state}-\eqref{equ:linear_measure}, by choosing the latest best state estimates $\hat{\mathbf{x}}_{k|k-1}$ and $\hat{\mathbf{x}}_{k-1|k-1}$ as the linearization points, we can obtain the \emph{estimator's linearized system} as follows
\begin{align} \label{equ:linear_state_ekf}
\hat{\mathbf{e}}_{k} &= \hat{\mathbf{F}}_{k-1} \ \hat{\mathbf{e}}_{k-1} + \hat{\mathbf{G}}_{k-1} \ \mathbf{v}_k, \\
\tilde{\mathbf{y}}_{k} &= \hat{\mathbf{H}}_{k} \ \hat{\mathbf{e}}_{k} + \hat{\mathbf{D}}_k \  \mathbf{w}_k, 
\label{equ:linear_measure_ekf}
\end{align} 
where $\hat{\mathbf{F}}_{k-1}$, $\hat{{\mathbf{G}}}_{k-1}$, $\hat{{\mathbf{H}}}_{k}$ and $\hat{{\mathbf{D}}}_{k}$ are derived from replacing the nominal linearization points $\mathbf{x}^{*}_{k-1}$ and $\mathbf{x}^{*}_{k}$ with the current best state estimate and prediction, i.e.,  
\begin{align} \label{equ:prop_jac_2}
&\hat{\mathbf{F}}_{k-1} = \mathbf{F}(\hat{\mathbf{x}}_{k-1|k-1}) , \\
\label{equ:noise_jac_2}
&\hat{\mathbf{G}}_{k-1} =  \mathbf{G}(\hat{\mathbf{x}}_{k-1|k-1}) , \\
\label{equ:update_jac_2}
&\hat{\mathbf{H}}_{k} = \mathbf{H} (\hat{\mathbf{x}}_{k|k-1}) , \\
\label{equ:m_noise_jac_2}
&\hat{\mathbf{D}}_{k} = \mathbf{D}  (\hat{\mathbf{x}}_{k|k-1}) .
\end{align}

For the estimator’s linearized system \eqref{equ:linear_state_ekf}-\eqref{equ:linear_measure_ekf}, we also define its local observability matrix over a time interval $[k, k+\ell]$ as follows
\begin{equation}
\hat{\boldsymbol{\mathcal{O}}}_k
\triangleq
\left[
\begin{array}{c}
\hat{\mathbf{H}}_{k} \\
\hat{\mathbf{H}}_{k+1} \hat{\mathbf{F}}_{k} \\
\vdots \\
\hat{\mathbf{H}}_{k+\ell} \hat{\mathbf{F}}_{k+\ell-1} \cdots \hat{\mathbf{F}}_{k} \\
\end{array}
\right]
\label{equ:ekf_obs_mat}
\end{equation} 
with $\ell \rightarrow \infty$. 
As aforementioned, the unobservable subspace of the estimator's linearized system is the kernel space of the local observable matrix $\hat{\boldsymbol{\mathcal{O}}}_k$, denoted as $\Ker(\hat{\boldsymbol{\mathcal{O}}}_k)$. A significant challenge faced by the estimator is that the estimator system's unobservable subspace does not align with that of the nominal linearized system \eqref{equ:linear_state}-\eqref{equ:linear_measure}. This occurs in various practical applications and leads to inconsistency issues. In what follows, we will analyze the properties of the estimator system's unobservable subspace and establish its relationship to that of the nominal linearized system \eqref{equ:linear_state}-\eqref{equ:linear_measure}.

\subsection{Observability Matching Conditions}
Before conducting the theoretical analysis, we need to provide essential mathematical preliminaries on the unobservable subspace of linear time-varying systems. Consider a general linear time-varying system governed by the following discrete differential equations:
\begin{align} \label{equ:general_dynam}
\mathbf{x}_{k} &= \mathbb{A}_{k-1} \ \mathbf{x}_{k-1} + \mathbb{B}_{k-1} \ \mathbf{u}_k \\
\mathbf{y}_{k} &= \mathbb{C}_{k} \ \mathbf{x}_{k} \label{equ:general_measurement}
\end{align}
where $\mathbb{A}_{k-1}$, $\mathbb{B}_{k-1}$ and $\mathbb{C}_{k}$ are time-dependent. Let $\mathbb{O}_k$ denote the local observability matrix of system \eqref{equ:general_dynam}-\eqref{equ:general_measurement}. In light of Theorem 4.4 in \cite{B68}, the unobservable subspace $\Ker(\mathbb{O}_k)$ of system \eqref{equ:general_dynam}-\eqref{equ:general_measurement} is \emph{forward $\mathbb{A}_k$-invariant}, i.e.,
\begin{align} \label{equ:invariant1}
\mathbb{A}_{k} \Ker(\mathbb{O}_k) \subseteq \Ker(\mathbb{O}_{k+1}) .
\end{align}
Moreover, according to Proposition 4.5 in \cite{B68}, the unobservable subspace $\Ker(\mathbb{O}_k)$ of system \eqref{equ:general_dynam}-\eqref{equ:general_measurement} belongs to the kernel space of the measurement Jacobian matrix $\mathbb{C}_{k}$, i.e.,
\begin{align} \label{equ:invariant2}
\Ker(\mathbb{O}_k) \subseteq \Ker(\mathbb{C}_{k}).
\end{align}
For any linear time-varying systems, the two properties given by \eqref{equ:invariant1}-\eqref{equ:invariant2} always hold. 

Subsequently, we will perform the theoretical analysis to answer the following well-concerned questions:
\begin{itemize}
\item \textbf{Q1}: What systems incur observability mismatch issues?
\item \textbf{Q2}: What properties of the estimator system lead to observability mismatch?
\item \textbf{Q3}: What is the relationship between the estimator system's unobservable subspace and that of the nominal linearized system?
\item \textbf{Q4}: Under what conditions does the unobservable subspace of the estimator system align with that of the nominal linearized system?
\end{itemize}


To begin with, we detail the concept of observability mismatch. To this end, we first analyze the unobservable subspace of the nominal linearized system \eqref{equ:linear_state}-\eqref{equ:linear_measure}, and clarify that any unobservable vector of this system is a vector-valued function of the nominal linearization point, which is summarized in the following lemma.

\begin{lemma} \label{propoition:pre1}
Any element of the unobservable subspace $\Ker (\boldsymbol{\mathcal{O}}_{k})$ of the nominal linearized system \eqref{equ:linear_state}-\eqref{equ:linear_measure} is a vector-valued function of the nominal linearization point $\mathbf{x}_{k}^{\ast}$. 
\end{lemma}

\begin{proof}
Please see Appendix \ref{app:propoition:pre1}.
\end{proof}

%
According to Lemma \ref{propoition:pre1}, the basis of the unobservable subspace $\Ker(\boldsymbol{\mathcal{O}}_k)$, which is a column-wise stack of basis vectors of $\Ker(\boldsymbol{\mathcal{O}}_k)$, is a matrix-valued function \cite{B79} of the nominal linearization point $\mathbf{x}_{k}^{\ast}$, represented as $\mathbf{B}_{k}(\mathbf{x}_{k}^{\ast})$. In contrast, for the estimator's linearized system \eqref{equ:linear_state_ekf}-\eqref{equ:linear_measure_ekf}, the basis of its unobservable subspace $\Ker (\hat{\boldsymbol{\mathcal{O}}}_{k})$ should be a matrix-valued function of the predicted state $\hat{\mathbf{x}}_{k|k-1}$. Specifically, according to the local observability matrix \eqref{equ:ekf_obs_mat}, the unobservable subspace $\Ker (\hat{\boldsymbol{\mathcal{O}}}_{k})$ should belong to the kernel space $\Ker (\hat{\mathbf{H}}_{k})$ of the linearized measurement Jacobian matrix $\hat{\mathbf{H}}_{k}$. Since $\hat{\mathbf{H}}_{k}$ is a matrix-valued function of the predicted state $\hat{\mathbf{x}}_{k|k-1}$, i.e., 
$$
\hat{\mathbf{H}}_{k}=\mathbf{H}(\hat{\mathbf{x}}_{k|k-1}) ,
$$
it follows that the basis matrix of $\Ker (\hat{\mathbf{H}}_{k})$ is a matrix-valued function of the predicted state $\hat{\mathbf{x}}_{k|k-1}$. Thus the basis of the estimator system's unobservable subspace $\Ker(\hat{\boldsymbol{\mathcal{O}}}_k)$ is also a matrix-valued function of the predicted state $\hat{\mathbf{x}}_{k|k-1}$. 

Ideally, the estimator system's unobservable subspace $\Ker (\hat{\boldsymbol{\mathcal{O}}}_{k})$ should coincide with that of the nominal linearized system \eqref{equ:linear_state}-\eqref{equ:linear_measure}. In other words, the basis of the unobservable subspaces for both systems should be identical functions, albeit evaluated at different linearization points. Otherwise, the estimator system's unobservable subspace does not align with that of the nominal linearized system. To make it more clear, we give the formal definition as follows. 

\begin{definition}
Let $\mathbf{B}_{k}(\mathbf{x}_{k}^{\ast})$ denote the basis matrix of the nominal linearized system's unobservable subspace $\Ker (\boldsymbol{\mathcal{O}}_{k})$, i.e.,
$$
\Ker (\boldsymbol{\mathcal{O}}_{k}) = \Span(\mathbf{B}_{k}(\mathbf{x}_{k}^{\ast})). 
$$
The estimator system's unobservable subspace is said to coincide with that of the nominal linearized system if
$$
\Ker (\hat{\boldsymbol{\mathcal{O}}}_{k}) = \Span(\mathbf{B}_{k}(\hat{\mathbf{x}}_{k|k-1})). 
$$
Otherwise, the two system's unobservable subspace does not match.
\end{definition}

It should be noted that except for explicitly relying on the linearization point $\mathbf{x}_{k}^{\ast}$, the basis matrix may also be a state-independent (constant) matrix. To distinguish the two cases, at any time step $k$, the basis matrix is denoted as $\mathbf{B}_{k}(\mathbf{x}_{k}^{\ast})$ when it explicitly relies on the linearization point $\mathbf{x}_{k}^{\ast}$; otherwise, it is simply represented as $\mathbf{B}_{k}$. The subscript $k$ is used to distinguish the basis matrix at different moments since its dimension might be varying. For clarity, we give a formal definition.
\begin{definition} \label{def_1}
An unobservable subspace $\Ker(\boldsymbol{\mathcal{O}}_k)$ is said to be state-independent (constant) if there exists a state-independent (constant) basis matrix $\mathbf{B}_{k}$ such that
$$
\Ker(\boldsymbol{\mathcal{O}}_k)
=
\Span(\mathbf{B}_{k}) ,
$$
otherwise, it is called state-dependent.
\end{definition}

The property presented in Definition \ref{def_1} adapts to both the nominal linearized system and the estimator system. Next, we analyze the properties of the unobservable subspace of the estimator system and its relationship to that of the nominal linearized system, and establish a necessary and sufficient condition for ensuring observability matching.



One might think that the unobservable subspace of the estimator system resembles that of typical linear time-varying systems, consisting of both state-dependent and state-independent components. Surprisingly, we find it is always state-independent (constant). We summarize this finding in the following theorem. 

\begin{theorem} \label{theorem_1_invariant}
The unobservable subspace $\Ker(\hat{\boldsymbol{\mathcal{O}}}_k)$ of the estimator's linearized system \eqref{equ:linear_state_ekf}-\eqref{equ:linear_measure_ekf} is always state-independent (constant). 
\end{theorem}

\begin{proof}
Please see Appendix \ref{app:theorem_1_invariant}.
\end{proof}

Now you may want to know the relationship between the unobservable subspace of the estimator's linearized system \eqref{equ:linear_state_ekf}-\eqref{equ:linear_measure_ekf} and that of the nominal linearized system \eqref{equ:linear_state}-\eqref{equ:linear_measure}. We present it in the following theorem. 


\begin{theorem} \label{theorem_2_invariant}
The unobservable subspace of the estimator's linearized system \eqref{equ:linear_state_ekf}-\eqref{equ:linear_measure_ekf} belongs to that of the nominal linearized system \eqref{equ:linear_state}-\eqref{equ:linear_measure}, i.e.,
\begin{equation}
\Ker (\hat{\boldsymbol{\mathcal{O}}}_{k}) \subseteq \Ker (\boldsymbol{\mathcal{O}}_{k}) .
\end{equation}
\end{theorem}

\begin{proof}
Please see Appendix \ref{app:theorem_2_invariant}.
\end{proof}



Now we are ready to establish the necessary and sufficient condition under which the unobservable subspace of the estimator's linearized system \eqref{equ:linear_state_ekf}-\eqref{equ:linear_measure_ekf} is equivalent to that of the nominal linearized system \eqref{equ:linear_state}-\eqref{equ:linear_measure}.

\begin{theorem} \label{theorem:final}
The unobservable subspace of the estimator's linearized system \eqref{equ:linear_state}-\eqref{equ:linear_measure} is exactly equivalent to that of the nominal linearized system \eqref{equ:linear_state_ekf}-\eqref{equ:linear_measure_ekf}, i.e.,
\begin{equation}
\Ker (\hat{\boldsymbol{\mathcal{O}}}_{k}) = \Ker (\boldsymbol{\mathcal{O}}_{k}) .
\end{equation}
if and only if the unobservable subspace of the nominal linearized system is state-independent (constant), i.e.,
$$
\Ker (\boldsymbol{\mathcal{O}}_{k}) 
=
\Span (\mathbf{B}_{k}) .
$$
\end{theorem}

\begin{proof}
Please see Appendix \ref{app:theorem:final}.
\end{proof}

\begin{remark}
According to the theoretical analysis from Theorem \ref{theorem_1_invariant} to Theorem \ref{theorem:final}, it can be concluded that the estimator system's unobservable subspace is always constant and belongs to the nominal linearized system's unobservable subspace. Therefore, the estimator incurs observability mismatch if the actual system's unobservable subspace is state-dependent. To mitigate the inconsistency issue caused by observability mismatch, one natural idea is to construct a new equivalent system having constant unobservable subspace and perform filtering based on this system.
\end{remark}

\begin{remark}
There are many ways to achieve a system that maintains a constant unobservable subspace. For linearization-based estimators, e.g., EKF, a more convenient method is to introduce a linear time-varying transformation that converts the nominal linearized system into this form. Encouragingly, such transformations are always available (see Section \ref{sec:transformation_1}) and are straightforward to design, significantly simplifying the development of consistent estimators. Furthermore, the available transformations are generally not unique, enabling us to leverage various transformations to enhance performance in other aspects, such as computational efficiency.
\end{remark}

\section{Transformation-based Consistent Estimation Approach}
\label{sec:transformation}


In this section, we present a transformation-based estimation approach to mitigate the inconsistency problem caused by observability mismatch. Specifically, we first propose two design approaches of linear time-varying transformations to ensure the transformed system possesses a state-independent unobservable subspace. By leveraging the designed observability-preserved transformations, two equivalent estimators, referred to as T-EKF 1 and T-EKF 2, are presented to achieve consistent estimation results. Furthermore, we discuss the distinctions and connections of the two estimators in terms of their realizations and computational complexity.

\subsection{Design of Linear Time-Varying Transformations}
\label{sec:transformation_1}

Let
$
\mathbf{T}_{k} = \mathbf{T}(\mathbf{x}_{k}) 
$
denote an invertible time-varying transformation.
Recalling the nominal linearized system \eqref{equ:linear_state}-\eqref{equ:linear_measure}, applying the transformation matrix $\mathbf{T}_{k}$ to the error-state $\mathbf{e}_k$ yields:
\begin{equation} \label{equ:transformation}
\bar{\mathbf{e}}_k = \mathbf{T}_{k} \mathbf{e}_k .
\end{equation} 
Substituting \eqref{equ:transformation} into the nominal linearized system \eqref{equ:linear_state}-\eqref{equ:linear_measure}, we obtain the following \emph{transformed linearized system}:
\begin{align} \label{equ:linear_state_z}
\bar{\mathbf{e}}_{k} &= (\mathbf{T}_{k} \mathbf{F}_{k-1} \mathbf{T}_{k-1}^{-1}) \ \bar{\mathbf{e}}_{k-1} + (\mathbf{T}_{k} \mathbf{G}_{k-1}) \ \mathbf{v}_k , \\
\tilde{\mathbf{y}}_{k} &= (\mathbf{H}_{k} \mathbf{T}_{k}^{-1}) \ \bar{\mathbf{e}}_{k} + \mathbf{D}_k \ \mathbf{w}_k .
\label{equ:linear_measure_z}
 \end{align}

The aim of designing a transformation is to obtain a linearized system such that its unobservable subspace remains time-invariant and thus independent of the linearization points. To achieve this, we first establish the relationship between the unobservable subspace of the transformed linearized system and that of the original system, which is concluded in the following lemma.

\begin{lemma}[Lemma 1 in \cite{B69}] \label{lemma1}
Let $\mathbb{N}_{k}$ denote the unobservable subspace of the original linearized system \eqref{equ:linear_state}-\eqref{equ:linear_measure} and $\bar{\mathbb{N}}_{k}$ the unobservable subspace of the transformed linearized system \eqref{equ:linear_state_z}-\eqref{equ:linear_measure_z}, respectively. The unobservable subspace $\bar{\mathbb{N}}_{k}$ is exactly equivalent to applying the invertible transformation matrix $\mathbf{T}_k$ to the unobservable subspace $\mathbb{N}_k$, i.e.,
\begin{equation} \label{equ:null_space_relation}
{
\begin{aligned} 
\bar{\mathbb{N}}_k 
\triangleq \mathbf{T}_k \boxtimes \mathbb{N}_k 
= \Span_{\col}
\left( \mathbf{T}_k \times \mathbf{N}_k \right) 
\end{aligned}}
\end{equation}
where $\mathbf{N}_k$ is the basis matrix of the unobservable subspace $\mathbf{N}_k$ with
$$
\mathbb{N}_k = \Span_{\col} (\mathbf{N}_k) .
$$
\end{lemma}

Based on Lemma \ref{lemma1}, various transformations can be devised to convert the unobservable subspace as state-independent. In the following, we present two design approaches for linear time-varying transformations. The first approach analytically constructs a transformation based on the structure of the unobservable subspace, while the second approach achieves feasible transformations by transforming the state propagation Jacobian as constant. 

\subsubsection{Transformation Design Approach 1}

In the following theorem, we present a closed-form construction approach to achieve a transformation by which the transformed system can possess a state-independent unobservable subspace.

\begin{theorem} \label{theorem:trans_design_1}

Let $\mathbf{N}_{k} \in \mathbb{R}^{n \times r}$ denote the basis matrix of the unobservable subspace of system \eqref{equ:linear_state}-\eqref{equ:linear_measure} with $r$ the dimension of the unobservable subspace. Partition $\mathbf{N}_{k}$ as follows
$$
\begin{aligned}
\mathbf{N}_{k}
=
\left[
\begin{array}{cc}
\mathbf{N}_{1, k}^{\top} & \mathbf{N}_{2, k}^{\top} \\
\end{array}
\right]^{\top} 
\end{aligned}
$$
where $\mathbf{N}_{1, k} \in \mathbb{R}^{r \times r}$ and $\mathbf{N}_{2, k} \in \mathbb{R}^{(n-r) \times r}$. Then there exists a transformation matrix constructed as follows
\begin{equation} \label{equ:trans_cons_t}
\mathbf{T}_{k}
=
\left[
\begin{array}{cc}
\mathbf{N}_{1, k} & \mathbf{0}_{r \times (n-r)} \\
\mathbf{N}_{2, k} & \mathbf{I}_{(n-r)}
\end{array}
\right]^{-1} 
\end{equation}
such that the transformed system's unobservable subspace becomes state-independent.
\end{theorem}

\begin{proof}
In light of Lemma \ref{lemma1}, by adopting the designed transformation matrix \eqref{equ:trans_cons_t}, the unobservable subspace changes as state-independent, i.e.,
$$
\bar{\mathbb{N}}_{k}
=
\Span_{\col} \left(
\left[
\begin{array}{c}
\mathbf{I}_{r} \\
\mathbf{0}_{(n-r) \times r}
\end{array}
\right] \right) 
$$
which completes the proof.
\end{proof}

Theorem \ref{theorem:trans_design_1} proves that a feasible transformation for the estimation problem of partially observable systems can always be found (\textbf{Existence of Transformations}), and provides an analytical constructive approach for the design of such a transformation. The specific design process can be found in Sections \ref{sec:application_cl_trans} and \ref{sec:application_vins_trans}.


\subsubsection{Transformation Design Approach 2}

The following theorem shows that the observability mismatch issue can be mitigated if the designed transformation matrix renders the state propagation Jacobian matrix constant. This offers a new design method that achieves both consistency and enhanced computational efficiency.


\begin{theorem} \label{theorem:trans_design_2}
The unobservable subspace of the transformed linearized system \eqref{equ:linear_state_z}-\eqref{equ:linear_measure_z} is state-independent if the transformation matrix $\mathbf{T}_{k}$ changes the state propagation Jacobian matrix $\mathbf{F}_{k-1}$ as constant, i.e.,
$$
(\mathbf{T}_{k} \mathbf{F}_{k-1} \mathbf{T}_{k-1}^{-1})
=
{\rm Const.}
$$
\end{theorem}

\begin{proof}
Please see Appendix \ref{app:theorem:trans_design_2}.
\end{proof}


In practice, we can generally establish a transformation matrix that converts the state propagation Jacobian matrix into an identity matrix. This will significantly improve computational efficiency in covariance propagation. In subsequent applications, we will present detailed design processes (see Section \ref{sec:application_cl_trans} and \ref{sec:application_tt_trans}).

\subsection{Transformed Extended Kalman Filter 1}
\label{sec:tekf}

We now present the Transformed Extended Kalman Filter (T-EKF 1) to address the inconsistency problem arising from observability mismatch. 

\subsubsection{Propagation}

At each time step $k$, suppose one has the estimate $\hat{\mathbf{x}}_{k-1|k-1}$ achieved in the previous step and  the latest control input $\mathbf{u}_{k}$. The state prediction $\hat{\mathbf{x}}_{k|k-1}$ can then be obtained through the nonlinear state model \eqref{equ:dyanmics} as follows:

$$
\hat{\mathbf{x}}_{k|k-1} = \mathbf{f}(\hat{\mathbf{x}}_{k-1|k-1}, \: \mathbf{u}_{k}, \: \mathbf{0}).
$$ 

\subsubsection{Update with Transformed System}

Recall the transformed linearized system in \eqref{equ:linear_state_z}-\eqref{equ:linear_measure_z}. Choosing the latest best state estimates $\hat{\mathbf{x}}_{k|k-1}$ and $\hat{\mathbf{x}}_{k-1|k-1}$ as the linearization points, the estimator's linearized error-state system can be written as:
\begin{align} \label{equ:linear_state_z_ekf}
\bar{\mathbf{e}}_{k|k-1} &= \bar{\mathbf{F}}_{k-1} \ \bar{\mathbf{e}}_{k-1|k-1} + \bar{\mathbf{G}}_{k-1} \ \mathbf{v}_k , \\
\tilde{\mathbf{y}}_{k} &= \bar{\mathbf{H}}_{k} \ \bar{\mathbf{e}}_{k|k-1} + \bar{\mathbf{D}}_k \  \mathbf{w}_k ,
\label{equ:linear_measure_z_ekf}
\end{align} 
where the matrices $\bar{\mathbf{F}}_{k-1}$, $\bar{{\mathbf{G}}}_{k-1}$, $\bar{{\mathbf{H}}}_{k}$ and $\bar{{\mathbf{D}}}_{k}$ result from replacing the nominal linearization points $\mathbf{x}^{*}_{k-1}$ and $\mathbf{x}^{*}_{k}$ with the current best estimate $\hat{\mathbf{x}}_{k-1|k-1}$ and prediction $\hat{\mathbf{x}}_{k|k-1}$, i.e.,  
\begin{align} \label{equ:prop_jac}
&\bar{\mathbf{F}}_{k-1} = \big(\mathbf{T}_k \mathbf{F}_{k-1} \mathbf{T}_{k-1}^{-1} \big) |_{\mathbf{x}^{\ast}_{k} = \hat{\mathbf{x}}_{k|k-1}, \mathbf{x}^{\ast}_{k-1} = \hat{\mathbf{x}}_{k-1|k-1}} , \\
\label{equ:noise_jac}
&\bar{\mathbf{G}}_{k-1} = \big( \mathbf{T}_{k} \mathbf{G}_{k-1} \big)  |_{\mathbf{x}^{\ast}_{k} = \hat{\mathbf{x}}_{k|k-1}, \mathbf{x}^{\ast}_{k-1} = \hat{\mathbf{x}}_{k-1|k-1}} , \\
\label{equ:update_jac}
&\bar{\mathbf{H}}_{k} = \big( \mathbf{H}_{k} \mathbf{T}_{k}^{-1} \big) |_{\mathbf{x}^{\ast}_{k} = \hat{\mathbf{x}}_{k|k-1}} , \\ \label{equ:m_noise_jac}
&\bar{\mathbf{D}}_{k} = \big( \mathbf{D}_k \big) |_{\mathbf{x}^{\ast}_{k} = \hat{\mathbf{x}}_{k|k-1}} .
\end{align}

With $\bar{\mathbf{P}}_{k-1|k-1}$ achieved in previous step $k-1$, the covariance is propagated by:
\begin{equation}
\bar{\mathbf{P}}_{k|k-1}
=
\bar{{\mathbf{F}}}_{k-1} \bar{\mathbf{P}}_{k-1|k-1} \bar{{\mathbf{F}}}_{k-1}^{\top}
+
\bar{{\mathbf{G}}}_{k-1} \mathbf{Q}_{k} \bar{{\mathbf{G}}}_{k-1}^{\top} .
\label{equ:cov_prop}
\end{equation}
Combing \eqref{equ:cov_prop} with \eqref{equ:linear_state_z_ekf}-\eqref{equ:linear_measure_z_ekf} leads to the following updates to achieve $\bar{\mathbf{e}}_{k|k}$ and $\bar{\mathbf{P}}_{k|k}$:
\begin{align}
&{
\bar{\mathbf{e}}_{k|k}
= \bar{\mathbf{K}}_{k} (\mathbf{y}_{k} - \mathbf{h}(\hat{\mathbf{x}}_{k|k-1}, \mathbf{0})) } \\
&{
\bar{\mathbf{P}}_{k|k} 
= 
(\mathbf{I} - \bar{\mathbf{K}}_{k} \bar{{\mathbf{H}}}_{k}) \bar{\mathbf{P}}_{k|k-1} 
\label{equ:alg_cov_update}}
\end{align}
where $\bar{\mathbf{K}}_{k}$ denotes the information gain matrix along the transformed coordinates given by
\begin{equation} \label{equ:info_gain}
\bar{\mathbf{K}}_{k} 
= 
\bar{\mathbf{P}}_{k|k-1} \bar{{\mathbf{H}}}_{k}^{\top} (\bar{{\mathbf{H}}}_{k} \bar{\mathbf{P}}_{k|k-1} \bar{{\mathbf{H}}}_{k}^{\top} + \bar{{\mathbf{D}}}_{k} \mathbf{R}_{k} \bar{{\mathbf{D}}}_{k}^{\top} )^{-1} .
\end{equation}

\subsubsection{Inverse Transformation}
Since the error-state $\bar{\mathbf{e}}_{k|k}$ is expressed in the transformed coordinate system, an inverse coordinate transformation is necessary to convert the estimation results back into the original coordinates for state updating. By utilizing the introduced coordinate transformation matrix, we present the following two state update approaches:
\begin{itemize}
\item Exact state update: according to the error-state transformation given in \eqref{equ:transformation}, the exact state update equation can be written as
\begin{equation}
\hat{\mathbf{x}}_{k|k} 
= 
\hat{\mathbf{x}}_{k|k-1} + \hat{\mathbf{T}}_{k|k}^{-1} \bar{\mathbf{K}}_{k} (\mathbf{y}_{k} - \mathbf{h}(\hat{\mathbf{x}}_{k|k-1}, \mathbf{0}))  
\label{equ:alg_state_transform}
\end{equation}
with $\hat{\mathbf{T}}_{k|k} = \mathbf{T}(\hat{\mathbf{x}}_{k|k})$.
In practice, we can seek for a closed-form solution of \eqref{equ:alg_state_transform} to obtain the optimal state estimate $\hat{\mathbf{x}}_{k|k}$.
\item Approximate state update: by replacing $\hat{\mathbf{T}}_{k|k}$ as its prediction $\hat{\mathbf{T}}_{k|k-1}$,
the mean estimate $\hat{\mathbf{x}}_{k|k}$ can be approximately obtained via
\begin{equation}
\hat{\mathbf{x}}_{k|k} 
= 
\hat{\mathbf{x}}_{k|k-1} + \hat{\mathbf{T}}_{k|k-1}^{-1} \bar{\mathbf{K}}_{k} (\mathbf{y}_{k} - \mathbf{h}(\hat{\mathbf{x}}_{k|k-1}, \mathbf{0}))  
\label{equ:alg_state_transform_approx}
\end{equation}
with $\hat{\mathbf{T}}_{k|k-1} = \mathbf{T}(\hat{\mathbf{x}}_{k|k-1})$.
\end{itemize}

\begin{remark}
The approximate solution provides an easy-to-implement manner for state update but with a slight sacrifice in performance. The degradation would be exaggerated when there is a significant discrepancy between the predicted and estimated values. To achieve better performance, we recommend to adopt the exact update approach. 
\end{remark}
 
\subsubsection{Initialization}

To initiate the filtering process, an initial guess for the state and covariance has to be provided. Typically, the initial state $\hat{\mathbf{x}}_{0|0}$ and covariance $\mathbf{P}_{0|0}$ are provided as a prior with
$$
\begin{aligned}
&\qquad \qquad \hat{\mathbf{x}}_{0|0} = \mathbb{E}(\mathbf{x}_0) , \\
&\hat{\mathbf{P}}_{0|0} =  \mathbb{E}[(\mathbf{x}_0 - \hat{\mathbf{x}}_{0|0})(\mathbf{x}_0 - \hat{\mathbf{x}}_{0|0})^\top] .
\end{aligned}
$$
As for the initial covariance $\bar{\mathbf{P}}_{0|0}$ of the transformed error-state $\bar{\mathbf{e}}_{0|0}$, according to the transformation equation \eqref{equ:transformation}, we have
\begin{equation} \label{equ:cov_transform}
\begin{aligned}
\bar{\mathbf{P}}_{0|0} =  \hat{\mathbf{T}}_{0|0} \hat{\mathbf{P}}_{0|0} \hat{\mathbf{T}}_{0|0}^{\top} 
\end{aligned}
\end{equation}
where $\hat{\mathbf{T}}_{0|0}$ is an estimate of $\mathbf{T}_0$ given by
$$
\hat{\mathbf{T}}_{0|0} 
= \mathbf{T}(\hat{\mathbf{x}}_{0|0}) .
$$
The initialization step is only executed at the onset of state estimation. The remaining steps are then implemented recursively. For clarity, we summarize the overall procedures of T-EKF 1 in Alg. \ref{alg:a1}. 

{\begin{algorithm}
\caption{Transformed EKF 1} \label{alg:a1}
\phantomsection \textbf{Initialization:}

\enspace $\hat{\mathbf{x}}_{0|0} = \mathbb{E}(\mathbf{x}_0)$

\enspace  $\hat{\mathbf{P}}_{0|0} =  \mathbb{E}[(\mathbf{x}_0 - \hat{\mathbf{x}}_{0|0})(\mathbf{x}_0 - \hat{\mathbf{x}}_{0|0})^\top] $

\enspace $\bar{\mathbf{P}}_{0|0} =  \hat{\mathbf{T}}_{0|0} \hat{\mathbf{P}}_{0|0} \hat{\mathbf{T}}_{0|0}^{\top} $

\phantomsection \textbf{Loop:}

\enspace \textbf{State propagation:}

\quad $\hat{\mathbf{x}}_{k|k-1} = \mathbf{f}(\hat{\mathbf{x}}_{k-1|k-1}, \mathbf{u}_{k}, \mathbf{0})$

\enspace \textbf{Jacobians computation:}

\enspace  $\ \ \bar{{\mathbf{F}}}_{k-1} \! = \! \big( \mathbf{T}_k \mathbf{F}_{k-1} \mathbf{T}_{k-1}^{-1} \big) |_{\mathbf{x}^{\ast}_{k} = \hat{\mathbf{x}}_{k|k-1}, \mathbf{x}^{\ast}_{k-1} = \hat{\mathbf{x}}_{k-1|k-1}}$

\quad $\bar{{\mathbf{G}}}_{k-1} = \big( \mathbf{T}_{k} \mathbf{G}_{k-1} \big) |_{\mathbf{x}^{\ast}_{k} = \hat{\mathbf{x}}_{k|k-1}, \mathbf{x}^{\ast}_{k-1} = \hat{\mathbf{x}}_{k-1|k-1}}$

\quad $\bar{{\mathbf{H}}}_k = \big( \mathbf{H}_{k} \mathbf{T}_{k}^{-1} \big) |_{\mathbf{x}^{\ast}_{k} = \hat{\mathbf{x}}_{k|k-1}}$

\quad $\bar{{\mathbf{D}}}_k = \big( \mathbf{D}_k \big) |_{\mathbf{x}^{\ast}_{k} = \hat{\mathbf{x}}_{k|k-1}}$

\enspace \textbf{Measurement update:}

\quad $\bar{\mathbf{P}}_{k|k-1} = \bar{{\mathbf{F}}}_{k-1} \bar{\mathbf{P}}_{k-1|k-1} \bar{{\mathbf{F}}}_{k-1}^{\top} + \bar{\mathbf{G}}_{k-1} \mathbf{Q}_{k} \bar{\mathbf{G}}_{k-1}^{\top} $

\quad $\bar{\mathbf{K}}_{k} 
= 
\bar{\mathbf{P}}_{k|k-1} \bar{{\mathbf{H}}}_{k}^{\top} (\bar{{\mathbf{H}}}_{k} \bar{\mathbf{P}}_{k|k-1} \bar{{\mathbf{H}}}_{k}^{\top} + \bar{{\mathbf{D}}}_{k} \mathbf{R}_{k} \bar{{\mathbf{D}}}_{k}^{\top} )^{-1}$

\enspace  \textbf{Exact state update}

\quad $ \hat{\mathbf{x}}_{k|k} 
= 
\hat{\mathbf{x}}_{k|k-1} + \hat{\mathbf{T}}_{k|k}^{-1} \bar{\mathbf{K}}_{k} (\mathbf{y}_{k} - \mathbf{h}(\hat{\mathbf{x}}_{k|k-1}, \mathbf{0}))  $

\enspace \textbf{OR Approximate state update}

\quad $ \hat{\mathbf{x}}_{k|k} 
= 
\hat{\mathbf{x}}_{k|k-1} + \hat{\mathbf{T}}_{k|k-1}^{-1} \bar{\mathbf{K}}_{k} (\mathbf{y}_{k} - \mathbf{h}(\hat{\mathbf{x}}_{k|k-1}, \mathbf{0}))  $

\enspace \textbf{Covariance update}

\quad $\bar{\mathbf{P}}_{k|k}  = 
(\mathbf{I} -  \bar{\mathbf{K}}_{k} \bar{{\mathbf{H}}}_{k}) \bar{\mathbf{P}}_{k|k-1}$

\end{algorithm}
\vspace{-1.5em}
}

\subsection{Transformed Extended Kalman Filter 2}
\label{sec:tekf2.0}
We subsequently present an equivalent version of the T-EKF 1, referred to as T-EKF 2. In comparison with the T-EKF 1, the equivalent filter conducts propagation and update operations in the original coordinates, and guarantees consistency by leveraging a state and covariance correction. 

\subsubsection{Propagation}

Similar to the T-EKF 1, when the estimate $\hat{\mathbf{x}}_{k-1|k-1}$ obtained in the previous step and the latest control input $\mathbf{u}_{k}$ are available, the state prediction $\hat{\mathbf{x}}_{k|k-1}$ can be achieved through the nonlinear state model \eqref{equ:dyanmics} as follows:
$$
\hat{\mathbf{x}}_{k|k-1} = \mathbf{f}(\hat{\mathbf{x}}_{k-1|k-1}, \: \mathbf{u}_{k}, \: \mathbf{0}).
$$ 

\subsubsection{Update with Original System}
Different from T-EKF 1, the equivalent filter updates the state and covariance based on the original linearized system in the original coordinates. To obtain the update equations of this equivalent filter, we first establish the relationship between the two estimators' linearized systems. Recalling the Jacobian matrices of the two estimators' linearized systems as described in \eqref{equ:prop_jac_2}-\eqref{equ:m_noise_jac_2} and \eqref{equ:prop_jac}-\eqref{equ:m_noise_jac}, respectively, we have
\begin{align} \label{equ:F_jac_trans}
&\bar{\mathbf{F}}_{k-1} = \hat{\mathbf{T}}_{k|k-1} \hat{\mathbf{F}}_{k-1} \hat{\mathbf{T}}_{k-1|k-1}^{-1} , \\
\label{equ:noise_jac_trans} 
&\bar{\mathbf{G}}_{k-1} = \hat{\mathbf{T}}_{k|k-1} \hat{\mathbf{G}}_{k-1}  , \\
\label{equ:update_jac_trans}
&\bar{\mathbf{H}}_{k} = \hat{\mathbf{H}}_{k} \hat{\mathbf{T}}_{k|k-1}^{-1} , \\ 
\label{equ:m_noise_jac_trans}
&\bar{\mathbf{D}}_{k} = \hat{\mathbf{D}}_k .
\end{align}
Moreover, let $\hat{\mathbf{P}}$ and $\bar{\mathbf{P}}$ denote the covariance estimates expressed in the original and transformed coordinates, respectively. In light of the error-state transformation equation \eqref{equ:transformation}, the two covariance expressions are related as follows:
\begin{equation} \label{equ:cov_equ}
\begin{aligned}
\bar{\mathbf{P}}_{k|k} &= \hat{\mathbf{T}}_{k|k} \hat{\mathbf{P}}_{k|k} (\hat{\mathbf{T}}_{k|k})^\top , \\
\bar{\mathbf{P}}_{k|k-1} &=  \hat{\mathbf{T}}_{k|k-1} \hat{\mathbf{P}}_{k|k-1} (\hat{\mathbf{T}}_{k|k-1})^\top .
\end{aligned}
\end{equation}
Then we derive the update equations of this equivalent filter. Specifically, substituting \eqref{equ:F_jac_trans}, \eqref{equ:noise_jac_trans} and \eqref{equ:cov_equ} into \eqref{equ:cov_prop} yields the following equivalent covariance propagation equation:
\begin{equation}
\hat{\mathbf{P}}_{k|k-1}
=
\hat{{\mathbf{F}}}_{k-1} \hat{\mathbf{P}}_{k-1|k-1} \hat{{\mathbf{F}}}_{k-1}^{\top}
+
\hat{{\mathbf{G}}}_{k-1} \mathbf{Q}_{k} \hat{{\mathbf{G}}}_{k-1}^{\top} .
\label{equ:cov_prop_2}
\end{equation}
Substituting \eqref{equ:update_jac_trans}, \eqref{equ:m_noise_jac_trans} and \eqref{equ:cov_equ} into \eqref{equ:info_gain}, 
we can attain the following update equation of the information gain matrix:
\begin{equation} \label{equ:info_gain2}
\hat{\mathbf{K}}_{k} 
= 
\hat{\mathbf{P}}_{k|k-1} \hat{{\mathbf{H}}}_{k}^{\top} (\hat{{\mathbf{H}}}_{k} \hat{\mathbf{P}}_{k|k-1} \hat{{\mathbf{H}}}_{k}^{\top} + \hat{{\mathbf{D}}}_{k} \mathbf{R}_{k} \hat{{\mathbf{D}}}_{k}^{\top} )^{-1} 
\end{equation}
with
\begin{equation} \label{equ:info_gain_trans}
\bar{\mathbf{K}}_{k} = \hat{\mathbf{T}}_{k|k-1} \hat{\mathbf{K}}_{k} . 
\end{equation}
According to the relationship established in \eqref{equ:info_gain_trans} as well as the T-EKF 1's state update equations given in \eqref{equ:alg_state_transform}-\eqref{equ:alg_state_transform_approx}, we can get two state update equations for T-EKF 2 as follows 
\begin{itemize}
\item Exact state update: substituting \eqref{equ:info_gain_trans} into \eqref{equ:alg_state_transform} results in the following equivalent exact state update equation 
\begin{equation}
\begin{aligned}
\hat{\mathbf{x}}_{k|k} 
= 
\hat{\mathbf{x}}_{k|k-1} + \Delta \mathbf{T}_{k} \hat{\mathbf{K}}_{k} (\mathbf{y}_{k} - \mathbf{h}(\hat{\mathbf{x}}_{k|k-1}, \mathbf{0}))  
\end{aligned}
\label{equ:alg_state_transform2}
\end{equation}
with
$
\Delta \mathbf{T}_{k} \: = \: (\hat{\mathbf{T}}_{k|k})^{-1} \hat{\mathbf{T}}_{k|k-1} 
$
called \emph{correction factor}.
\item Approximate state update: by substituting \eqref{equ:info_gain_trans} into \eqref{equ:alg_state_transform_approx}, we acquire the following equivalent approximate state update equation
\begin{equation}
\begin{aligned}
\hat{\mathbf{x}}_{k|k} 
= 
\hat{\mathbf{x}}_{k|k-1} + \hat{\mathbf{K}}_{k} (\mathbf{y}_{k} - \mathbf{h}(\hat{\mathbf{x}}_{k|k-1}, \mathbf{0}))  .
\end{aligned}
\label{equ:alg_state_transform_approx2}
\end{equation}
This approximate state update equation is exactly identical to that of the classical EKF. 
\end{itemize}
Furthermore, in terms of \eqref{equ:update_jac_trans}, \eqref{equ:cov_equ}, \eqref{equ:info_gain_trans} and \eqref{equ:alg_cov_update}, the covariance can be updated in the original coordinates as follows
\begin{equation}
\begin{aligned}
\hat{\mathbf{P}}_{k|k} 
= 
(\Delta \mathbf{T}_{k}) \hat{\mathbf{S}}_{k} (\Delta \mathbf{T}_{k})^{\top} 
\label{equ:alg_cov_update2}
\end{aligned}
\end{equation}
with $\Delta \mathbf{T}_{k}$ the correction factor and $\hat{\mathbf{S}}_{k}$ given by
$$
\hat{\mathbf{S}}_{k} = (\mathbf{I} - \hat{\mathbf{K}}_{k} \hat{{\mathbf{H}}}_{k}) \hat{\mathbf{P}}_{k|k-1} .
$$
The detailed procedures for this equivalent filter are summarized in Alg. \ref{alg:a2}.


{\begin{algorithm}
\caption{Transformed EKF 2} \label{alg:a2}
\phantomsection \textbf{Initialization:}

\enspace $\hat{\mathbf{x}}_{0|0} = \mathbb{E}(\mathbf{x}_0)$

\enspace $\hat{\mathbf{P}}_{0|0} =  \mathbb{E}[(\mathbf{x}_0 - \hat{\mathbf{x}}_{0|0})(\mathbf{x}_0 - \hat{\mathbf{x}}_{0|0})^\top] $


\phantomsection \textbf{Loop:}

\enspace \textbf{State propagation:}

\quad $\hat{\mathbf{x}}_{k|k-1} = \mathbf{f}(\hat{\mathbf{x}}_{k-1|k-1}, \mathbf{u}_{k}, \mathbf{0})$

\enspace \textbf{Jacobians computation:}

\enspace $\ \ \hat{{\mathbf{F}}}_{k-1} \! = \! \big( \mathbf{F}_{k-1} \big) |_{\mathbf{x}^{\ast}_{k} = \hat{\mathbf{x}}_{k|k-1}, \mathbf{x}^{\ast}_{k-1} = \hat{\mathbf{x}}_{k-1|k-1}}$

\quad $\hat{{\mathbf{G}}}_{k-1} = \big( \mathbf{G}_{k-1} \big) |_{\mathbf{x}^{\ast}_{k} = \hat{\mathbf{x}}_{k|k-1}, \mathbf{x}^{\ast}_{k-1} = \hat{\mathbf{x}}_{k-1|k-1}}$

\quad $\hat{{\mathbf{H}}}_k = \big( \mathbf{H}_{k} \big) |_{\mathbf{x}^{\ast}_{k} = \hat{\mathbf{x}}_{k|k-1}}$

\quad $\hat{{\mathbf{D}}}_k = \big( \mathbf{D}_k \big) |_{\mathbf{x}^{\ast}_{k} = \hat{\mathbf{x}}_{k|k-1}}$

\enspace \textbf{Measurement update:}

\quad $\hat{\mathbf{P}}_{k|k-1} = \hat{{\mathbf{F}}}_{k-1} \hat{\mathbf{P}}_{k-1|k-1} \hat{{\mathbf{F}}}_{k-1}^{\top} + \hat{{\mathbf{G}}}_{k-1} \mathbf{Q}_{k} \hat{{\mathbf{G}}}_{k-1}^{\top} $

\quad $\hat{\mathbf{K}}_{k} = \hat{\mathbf{P}}_{k|k-1} \hat{{\mathbf{H}}}_{k}^{\top} (\hat{{\mathbf{H}}}_{k} \hat{\mathbf{P}}_{k|k-1} \hat{{\mathbf{H}}}_{k}^{\top} + \hat{{\mathbf{D}}}_{k} \mathbf{R}_{k} \hat{{\mathbf{D}}}_{k}^{\top} )^{-1}$ 

\enspace \textbf{Exact state update:}

\quad $ \hat{\mathbf{x}}_{k|k} = \hat{\mathbf{x}}_{k|k-1} +  \Delta \mathbf{T}_{k} \hat{\mathbf{K}}_{k} (\mathbf{y}_{k} - \mathbf{h}(\hat{\mathbf{x}}_{k|k-1}, \mathbf{0})) $

\enspace \textbf{OR Approximate state update:}

\quad $ \hat{\mathbf{x}}_{k|k} 
= 
\hat{\mathbf{x}}_{k|k-1} + \hat{\mathbf{K}}_{k} (\mathbf{y}_{k} - \mathbf{h}(\hat{\mathbf{x}}_{k|k-1}, \mathbf{0})) $

\enspace \textbf{Covariance update:}

\quad $ \hat{\mathbf{S}}_{k} = (\mathbf{I} - \hat{\mathbf{K}}_{k} \hat{{\mathbf{H}}}_{k}) \hat{\mathbf{P}}_{k|k-1} $



\enspace \textbf{Covariance correction:}

\quad $ \hat{\mathbf{P}}_{k|k} = (\Delta \mathbf{T}_{k}) \hat{\mathbf{S}}_{k}  (\Delta \mathbf{T}_{k})^{\top} $

\end{algorithm}
\vspace{-2.0em}
}

\subsection{Comparison of T-EKF 1 and T-EKF 2}


\subsubsection{Implementation Process}
Although both estimators introduce transformations to address the issue of inconsistency, their implementation processes are completely distinct. Specifically, the T-EKF 1 develops a transformed linearized system that preserves correct observability properties and conducts filtering based on this transformed system. After that, the estimation results are converted back into the original coordinates to achieve the ultimate consistent estimates. In contrast, the T-EKF 2 utilizes the original linearized system for filtering, adhering to the classical EKF methodology.  The key to ensuring consistency lies in the incorporation of state and covariance corrections formed by transformation matrices. The two estimators perform filtering based on different linearized systems but yield identical estimation results. 

\subsubsection{Computational Efficiency}

Leveraging the duality of the two proposed filters, we can customize estimators for better computational efficiency. To illustrate this, we first analyze the computational complexity of both estimators. 
Assuming the dimensions of the state vector, the process noise vector, the measurement vector, and the measurement noise vector are $n$, $q$, $p$, and $p$, respectively, we analyze the computational complexity of both estimators in relation to the dimensions of these vectors. We mainly focus on the matrix multiplication operations and matrix assignment operations. The analysis results are summarized in Table \ref{table:complexity}. 

\begin{table*}[htbp] 
\centering
\caption{Comparison of computational complexity between T-EKF 1 and T-EKF 2. For clarity, higher complexity values are displayed in red, lower complexity values in blue, and equal values in black.}
\setlength{\abovecaptionskip}{0cm}
\setlength{\belowcaptionskip}{0cm}
\setlength{\tabcolsep}{0.08pt}
\renewcommand{\arraystretch}{1.5}
\begin{tabular}{|c|cc|cc|} 
\hline
& \multicolumn{2}{c|}{T-EKF 1} & \multicolumn{2}{c|}{T-EKF 2} \\
\hline
& Flow & \makecell{Complexity} & \makecell{Flow} & \makecell{Complexity} \\ \hline
1 & $\hat{\mathbf{x}}_{k|k-1} = \mathbf{f}(\hat{\mathbf{x}}_{k-1|k-1}, \mathbf{u}_{k}, \mathbf{0})$ & $\mathcal{O}(n)$ & $\hat{\mathbf{x}}_{k|k-1} = \mathbf{f}(\hat{\mathbf{x}}_{k-1|k-1}, \mathbf{u}_{k}, \mathbf{0})$ & $\mathcal{O}(n)$ \\  \hline
2 & $\bar{{\mathbf{F}}}_{k-1}  \! =  \! \big( \mathbf{T}_k \mathbf{F}_{k-1} \mathbf{T}_{k-1}^{-1} \big) |_{\mathbf{x}^{\ast}_{k} = \hat{\mathbf{x}}_{k|k-1}, \mathbf{x}^{\ast}_{k-1} = \hat{\mathbf{x}}_{k-1|k-1}}$ & $\mathcal{O}(n^{2})$ & $\hat{{\mathbf{F}}}_{k-1}  \! =  \! \big( \mathbf{F}_{k-1} \big) |_{\mathbf{x}^{\ast}_{k} = \hat{\mathbf{x}}_{k|k-1}, \mathbf{x}^{\ast}_{k-1} = \hat{\mathbf{x}}_{k-1|k-1}}$ & $\mathcal{O}(n^{2})$ \\  \hline
3 & $\bar{{\mathbf{G}}}_{k-1} = \big( \mathbf{T}_{k} \mathbf{G}_{k-1} \big) |_{\mathbf{x}^{\ast}_{k} = \hat{\mathbf{x}}_{k|k-1}, \mathbf{x}^{\ast}_{k-1} = \hat{\mathbf{x}}_{k-1|k-1}}$ & $\mathcal{O}(n q)$ & $\hat{{\mathbf{G}}}_{k-1} = \big( \mathbf{G}_{k-1} \big) |_{\mathbf{x}^{\ast}_{k} = \hat{\mathbf{x}}_{k|k-1}, \mathbf{x}^{\ast}_{k-1} = \hat{\mathbf{x}}_{k-1|k-1}}$ & $\mathcal{O}(n q)$ \\  \hline
4 & $\bar{\mathbf{P}}_{k|k-1}  \! =  \! \bar{{\mathbf{F}}}_{k-1} \bar{\mathbf{P}}_{k-1|k-1} \bar{{\mathbf{F}}}_{k-1}^{\top}  \! +  \! \bar{{\mathbf{G}}}_{k-1} \mathbf{Q}_{k} \bar{{\mathbf{G}}}_{k-1}^{\top} $ & $ \mathcal{O}(2n^{3}  \! +  \! n q^{2}  \! +  \! q n^{2}) $ & $ \hat{\mathbf{P}}_{k|k-1}  \! =  \! \hat{{\mathbf{F}}}_{k-1} \hat{\mathbf{P}}_{k-1|k-1} \hat{{\mathbf{F}}}_{k-1}^{\top}  \! +  \! \hat{{\mathbf{G}}}_{k-1} \mathbf{Q}_{k} \hat{{\mathbf{G}}}_{k-1}^{\top} $ & $\mathcal{O}(2n^{3}  \! +  \! n q^{2}  \! +  \! q n^{2})$ \\  \hline
5 & $\bar{{\mathbf{H}}}_k = \big( \mathbf{H}_{k} \mathbf{T}_{k}^{-1} \big) |_{\mathbf{x}^{\ast}_{k} = \hat{\mathbf{x}}_{k|k-1}}$ & $\mathcal{O}(n p)$ & $\hat{{\mathbf{H}}}_k = \big( \mathbf{H}_{k}  \big) |_{\mathbf{x}^{\ast}_{k} = \hat{\mathbf{x}}_{k|k-1}}$ & $\mathcal{O}(n p)$ \\  \hline
6 & $\bar{{\mathbf{D}}}_k = \big( \mathbf{D}_k \big) |_{\mathbf{x}^{\ast}_{k} = \hat{\mathbf{x}}_{k|k-1}}$ & $\mathcal{O}(p^{2})$ & $\hat{{\mathbf{D}}}_k = \big( \mathbf{D}_k \big) |_{\mathbf{x}^{\ast}_{k} = \hat{\mathbf{x}}_{k|k-1}}$ & $\mathcal{O}(p^{2})$ \\  \hline
7 & $ \bar{\mathbf{K}}_{k}  \! =  \! \bar{\mathbf{P}}_{k|k-1} \bar{{\mathbf{H}}}_{k}^{\top} (\bar{{\mathbf{H}}}_{k} \bar{\mathbf{P}}_{k|k-1} \bar{{\mathbf{H}}}_{k}^{\top}  \! +  \! \bar{{\mathbf{D}}}_{k} \mathbf{R}_{k} \bar{{\mathbf{D}}}_{k}^{\top} )^{-1} $ & $\mathcal{O}(2nr(n   \! +  \! p)  \! +  \! 3p^{3})$ & $\hat{\mathbf{K}}_{k}  \! =  \! \hat{\mathbf{P}}_{k|k-1} \hat{{\mathbf{H}}}_{k}^{\top} (\hat{{\mathbf{H}}}_{k} \hat{\mathbf{P}}_{k|k-1} \hat{{\mathbf{H}}}_{k}^{\top}  \! +  \! \hat{{\mathbf{D}}}_{k} \mathbf{R}_{k} \hat{{\mathbf{D}}}_{k}^{\top} )^{-1} $  & $\mathcal{O}(2nr(n   \! +  \! p)  \! +  \! 3p^{3})$ \\ \hline
8 & \makecell{Exact state update \\ $ \hat{\mathbf{x}}_{k|k}  \! =  \! \hat{\mathbf{x}}_{k|k-1}  \! +  \! \hat{\mathbf{T}}_{k|k}^{-1} \bar{\mathbf{K}}_{k} (\mathbf{y}_{k}  \! -  \! \mathbf{h}(\hat{\mathbf{x}}_{k|k-1}, \mathbf{0}))$} & \color{blue} $\mathcal{O}(n(n  \! +  \! p))$ & \makecell{Exact state update \\ $ \hat{\mathbf{x}}_{k|k}  \! =  \! \hat{\mathbf{x}}_{k|k-1}  \! +  \! \Delta \mathbf{T}_{k} \hat{\mathbf{K}}_{k} (\mathbf{y}_{k}  \! -  \! \mathbf{h}(\hat{\mathbf{x}}_{k|k-1}, \mathbf{0}))$} & \color{red} $\mathcal{O}(n(n^2 \! + \! n \! + \! p))$ \\  \hline
9 & \makecell{Or approximate state update \\  $ \hat{\mathbf{x}}_{k|k}  \! =  \! \hat{\mathbf{x}}_{k|k-1}  \! +  \! \hat{\mathbf{T}}_{k|k-1}^{-1} \bar{\mathbf{K}}_{k} (\mathbf{y}_{k}  \! -  \! \mathbf{h}(\hat{\mathbf{x}}_{k|k-1}, \mathbf{0}))$} & \color{red} $\mathcal{O}(n(n  \! +  \! p))$ & \makecell{Or approximate state update \\ $ \hat{\mathbf{x}}_{k|k}  \! =  \! \hat{\mathbf{x}}_{k|k-1}  \! +  \! \hat{\mathbf{K}}_{k} (\mathbf{y}_{k}  \! -  \! \mathbf{h}(\hat{\mathbf{x}}_{k|k-1}, \mathbf{0}))$} & \color{blue} 
 $\mathcal{O}(np)$ \\  \hline
10 & $ \bar{\mathbf{P}}_{k|k}  \! =  \! (\mathbf{I} \! - \! \bar{\mathbf{K}}_{k} \bar{{\mathbf{H}}}_{k}) \bar{\mathbf{P}}_{k|k-1} $ & $\mathcal{O}(n^{2}( n \! +  \!  p))$ & $ \hat{\mathbf{S}}_{k|k}  \! =  \! (\mathbf{I} \! - \! \hat{\mathbf{K}}_{k} \hat{{\mathbf{H}}}_{k}) \hat{\mathbf{P}}_{k|k-1}  $  &  $\mathcal{O}(n^{2}( n \! +  \!  p))$ \\  \hline
11 &  --- &  ---  & $\hat{\mathbf{P}}_{k|k} = (\Delta \mathbf{T}_{k}) \hat{\mathbf{S}}_{k}  (\Delta \mathbf{T}_{k})^{\top} $ & \color{red}  $\mathcal{O}(3n^3)$ \\
\hline
\end{tabular}
\label{table:complexity}
\vspace{-1.5em}
\end{table*}

As observed from Table \ref{table:complexity}, the two estimators distinguish from each other primarily in two key aspects: computation of Jacobian matrices (steps 2-3 and 5-6) and update of the state and covariance (steps 8-11). Since the closed-form expressions for the transformed Jacobian matrices are established in advance, it is only necessary to update the elements of these matrices during the execution process. Consequently, the computational cost associated with Jacobian matrix computation remains the same. During the state and covariance update step, the computational complexity of the T-EKF 2 is slightly higher than that of the T-EKF 1, primarily due to the additional operations required for state and covariance correction. 


\begin{remark}
The T-EKF 1 is preferable if the application is developed from scratch due to better computational efficiency. However, if inheriting from an existing framework, the T-EKF 2 is more recommended because it is easier to embed into those frameworks.
\end{remark}

\subsubsection{Combination Application}


It should be noted that the actual computation costs might differ significantly when choosing different coordinates. This depends on the structure of the actual Jacobian matrices. For instance, in subsequent applications (see Section \ref{sec:application_cl_trans}), the state propagation Jacobian matrix is changed into an identity matrix, substantially reducing the computation cost of covariance propagation. In practice, we can select the most suitable coordinates for filtering, wherein both the inconsistency and efficiency issues can be tackled simultaneously. It is worth mentioning that if two or more distinct coordinate transformations are necessary to address inconsistency and efficiency issues simultaneously, the proposed two filters can be combined to achieve consistent estimation while maintaining high computational efficiency. The discussions regarding the combined use of estimators are not the focus of this paper; therefore, we will not emphasize this topic further.

\section{Application in Multi-Robot Cooperative Localization}
\label{sec:application_cl}
In this section, we conduct application experiments on a typical partially observable system, i.e., multi-robot cooperative localization (CL) system, to verify the proposed method. We compare our method against state-of-the-art baselines in both simulation and real-world experiments and evaluate the performance under different coordinate transformations and state update approaches.

\subsection{Problem Statement of CL}

We consider a cooperative localization scenario where $m$ robots move in a two-dimensional plane. Each robot is equipped with proprioceptive sensors to sense its ego-motion information, and exteroceptive sensors to access relative robot-to-robot measurements. The objective is to jointly estimate these robots' poses with respect to a common reference frame by fusing their ego-motion information and relative measurements. Since absolute measurements are unavailable, the system is not fully observable, and the classical EKF suffers from inconsistency issues. Next, we give the motion and measurement model of this system.

Suppose the $m$ robots are homogeneous. For each robot $i$, the dynamics is described by
\begin{align} \label{equ:cl_system_i0}
\mathbf{p}_{i, k+1} &= \mathbf{p}_{i, k} + \mathbf{R}({\psi_{i, k}}) \Big(\mathbf{v}_{i, k} + \boldsymbol{\nu}_{i, k} \Big) \delta t , \\
\psi_{i, k+1} &= \psi_{i, k} + \Big(\omega_{i, k} + \varpi_{i, k} \Big) \delta t ,
\label{equ:cl_system_i}
\end{align}
where $k$ denotes the discrete-time index; $\mathbf{p}_{i, k} \in \mathbb{R}^2$ and $\psi_{i, k} \in \mathbb{R}$ denote the position and orientation of robot $i$ in a global reference frame, respectively;
$\mathbf{v}_{i, k} \in \mathbb{R}^2$ and $\omega_{i, k} \in \mathbb{R}$ denote robot $i$'s linear and angular velocities expressed relative to the body-fixed reference frame; $\boldsymbol{\nu}_{i, k}$ and $\varpi_{i, k}$ denote the input noises of proprioceptive sensors modeled as zero-mean white Gaussian; $\delta t$ is the sampling period, and $\mathbf{R}(\cdot) \in \mathbb{R}^{2\times 2}$ denotes the rotation matrix. Each robot $i$ might detect the relative position from its neighboring robot $j$, $j \in \{1,2,\cdots,m\} \slash \ 
 \{i\}$. The observation model of the relative position measurement can be written as follows
\begin{equation}
\mathbf{y}_{ij, k} =
\mathbf{R}({\psi_{i, k}})^{\top} \Big( \mathbf{p}_{j, k} - \mathbf{p}_{i, k} \Big) + \mathbf{w}_{ij, k}
\label{equ:cl_relative_position}
\end{equation} 
where $\mathbf{w}_{ij, k} \in \mathbb{R}^{2}$ denotes the measurement noise assumed to be subject to a zero-mean Gaussian distribution.

\subsection{Transformation Design}
\label{sec:application_cl_trans}

In the following, we will present the design of coordinate transformations to mitigate the inconsistency issue.
To this end, we first give the system's unobservable subspace. According to \cite{B7}, the unobservable subspace spans along the direction of global position and orientation, i.e., 
\begin{equation}
\mathbb{N}_k
=
\Span_{\col}
\left( 
\left[
\begin{array}{cc}
\mathbf{I}_{2} & \mathbf{J} \mathbf{p}_{1, {k}} \\
\mathbf{0}_{1 \times 2} &  1 \\
\vdots & \vdots \\
\mathbf{I}_{2} & \mathbf{J}\mathbf{p}_{m, {k}} \\
\mathbf{0}_{1 \times 2} & 1 \\
\end{array}
\right]
\right) 
\label{equ:CL_N_x}
\end{equation}
with $
\mathbf{J}
=
{\small
\left[
\begin{array}{cc}
0 & -1 \\
1 & 0  \\
\end{array}
\right] } .
$

\subsubsection{Transformation T1}

In light of Theorem \ref{theorem:trans_design_1} and the system's unobservable subspace in \eqref{equ:CL_N_x}, it is straightforward to construct the following transformation matrix 
\begin{equation} \label{equ:CL_T_inv_1}
{ 
\begin{aligned}
\mathbf{T}_k 
=
\left[
\setlength{\arraycolsep}{3.2pt}
\begin{array}{cc|ccc}
 \mathbf{I}_2 & \mathbf{J} \mathbf{p}_{1, k} & & &  \\
\mathbf{0}_{1 \times 2} & 1 & \multicolumn{3}{c}{\raisebox{1ex}[0pt]{\Large 0}} \\ \hline 
\vdots & \vdots  \\
\mathbf{I}_2 & \mathbf{J} \mathbf{p}_{m, k} \\
\mathbf{0}_{1 \times 2} & 1 & \multicolumn{3}{c}{\raisebox{3ex}[0pt]{\Large I}} \\
\end{array}
\right]^{-1} 
\end{aligned}}
\end{equation}
by which the transformed system's unobservable subspace becomes state-independent, thereby overcoming the inconsistency problem. 

\subsubsection{Transformation T2}
To improve computational efficiency, we present a new transformation design. Specifically, observing that the basis (see \eqref{equ:CL_N_x}) of the system's unobservable subspace is a column-wise stack of $m$ invertible block matrices, thus we can design a block-diagonal transformation matrix as follows:
\begin{equation} \label{equ:CL_T_inv_2}
{
\begin{aligned}
\mathbf{T}_k &= 
\left[
\setlength{\arraycolsep}{3.2pt}
\begin{array}{ccccc}
\mathbf{I}_2 & \mathbf{J} \mathbf{p}_{1, k} & \cdots & \mathbf{0}_{2 \times 2} & \mathbf{0}_{2 \times 1} \\
\mathbf{0}_{1 \times 2} & 1 & \cdots & \mathbf{0}_{1 \times 2} & 0  \\
\vdots & \vdots & \ddots & \vdots & \vdots   \\
\mathbf{0}_{2 \times 2} & \mathbf{0}_{2 \times 1} & \cdots & \mathbf{I}_2 & \mathbf{J} \mathbf{p}_{m, k}  \\
\mathbf{0}_{1 \times 2} & 0 & \cdots & \mathbf{0}_{1 \times 2} & 1 \\
\end{array}
\right]^{-1} .
\end{aligned}}
\end{equation}
Applying the transformation matrix \eqref{equ:CL_T_inv_2} makes the state propagation Jaobian matrix become an identity matrix. According to Theorem \ref{theorem:trans_design_2}, this would lead to a constant unobservable subspace.
Consequently, the inconsistency issue associated with cooperative localization can be mitigated.
Besides, the designed transformation \eqref{equ:CL_T_inv_2} significantly reduces the computational burden in covariance propagation. Subsequently, we will validate the performance of the proposed methods and transformations through both simulations and real-world experiments.
\subsection{Simulation Experiments}

Consider a simulation scenario where six robots move randomly in a two-dimensional plane with the linear velocity set to be $\rm 0.3 \: m/s$ and the angular velocity drawn from a uniform distribution $\rm [-0.1 rad/s, \: 0.1 rad/s]$. The ego-motion measurements (linear and angular velocities) of these robots are corrupted by zero-mean Gaussian white noises with standard deviations of ${\rm 0.15 \: m/s}$ and ${\rm 0.06 \: rad/s}$, respectively. In addition, each robot randomly detects relative position measurements to other robots with a detection probability of $20\%$. The standard deviation of relative position measurement noises is $\rm 0.1 \: m$. We compare the performance of the T-EKF with the state-of-the-art methods, including the Unscented Kalman filter (UKF), the classical EKF, the First-Estimate-Jacobian (FEJ) EKF \cite{B8}, and the Invariant EKF (I-EKF) \cite{B44}. 

\subsubsection{Evaluation of Consistency and Accuracy}
In this simulation test, we employ the Root Mean Square Error (RMSE) to quantify the accuracy of estimators, and the Normalized Estimation Error Squared (NEES) to evaluate estimation consistency. The simulation step is set to ${\rm 2.0 \: s}$ and $100$ trials are conducted. The evaluation results are summarized in Table~\ref{table:cl_sim}. The best results among these estimators are highlighted in bold and the second best results are printed in blue. Besides, we draw the plots of the average RMSE and NEES in Figure \ref{fig:cl_nees}. Notably, in this simulation scenario, for a consistent estimator, the average NEES should be close to the dimension of the state vector, i.e., the expected position NEES is equal to $2$ while the ideal orientation NEES approaches $1$.
\begin{table}[h] 
\centering
\caption{Average RMSE, NEES, and running time over Monte Carlo simulations for CL applications.}
\setlength{\abovecaptionskip}{0cm}
\setlength{\belowcaptionskip}{0cm}
\setlength{\tabcolsep}{4.6pt}
\renewcommand{\arraystretch}{1.3}
\begin{tabular}{|c|c|c|c|c|c|} 
\hline
Estimators & \makecell{RMSE \\ Pos. (m)} & \makecell{RMSE \\ Ori. (rad)} & \makecell{NEES \\ Pos.} & \makecell{NEES \\ Ori.} & \makecell{{ Running} \\ {Time (ms)}} \\
\hline
UKF & 4.330 & 0.460 & 5.153 & 3.518 & 4.681 \\
\hline
EKF & 5.104 & 0.591 & 6.828 & 8.233 & 1.556 \\
\hline
FEJ & 3.930 & 0.434 & 2.794 & 1.375 & {1.605} \\
\hline
I-EKF & 3.734 & 0.421 & 2.490 & 1.301 & 1.931 \\
\hline
T-EKF 1 (w/ T1) & \multirow{2}{*}{\textbf{3.618}} & \multirow{2}{*}{\color{blue}{0.410}} & \multirow{2}{*}{\textbf{2.392}} & \multirow{2}{*}{\color{blue}{1.258}} & 1.926 \\
\cline{1-1} \cline{6-6}
T-EKF 2 (w/ T1) & & & & & 2.165 \\
\hline
T-EKF 1 (w/ T2) & \multirow{2}{*}{\color{blue} {3.647}} & \multirow{2}{*}{\textbf{0.399}} & \multirow{2}{*}{\color{blue} {2.393}} & \multirow{2}{*}{\textbf{1.190}} & 1.854 \\
\cline{1-1} \cline{6-6}
T-EKF 2 (w/ T2) & & & & & 2.005 \\
\hline
\end{tabular}
\label{table:cl_sim}
\vspace{-1.0em}
\end{table}

\begin{figure}[!htp]
\centering
\subfloat[RMSE]{
\includegraphics[scale=0.48]{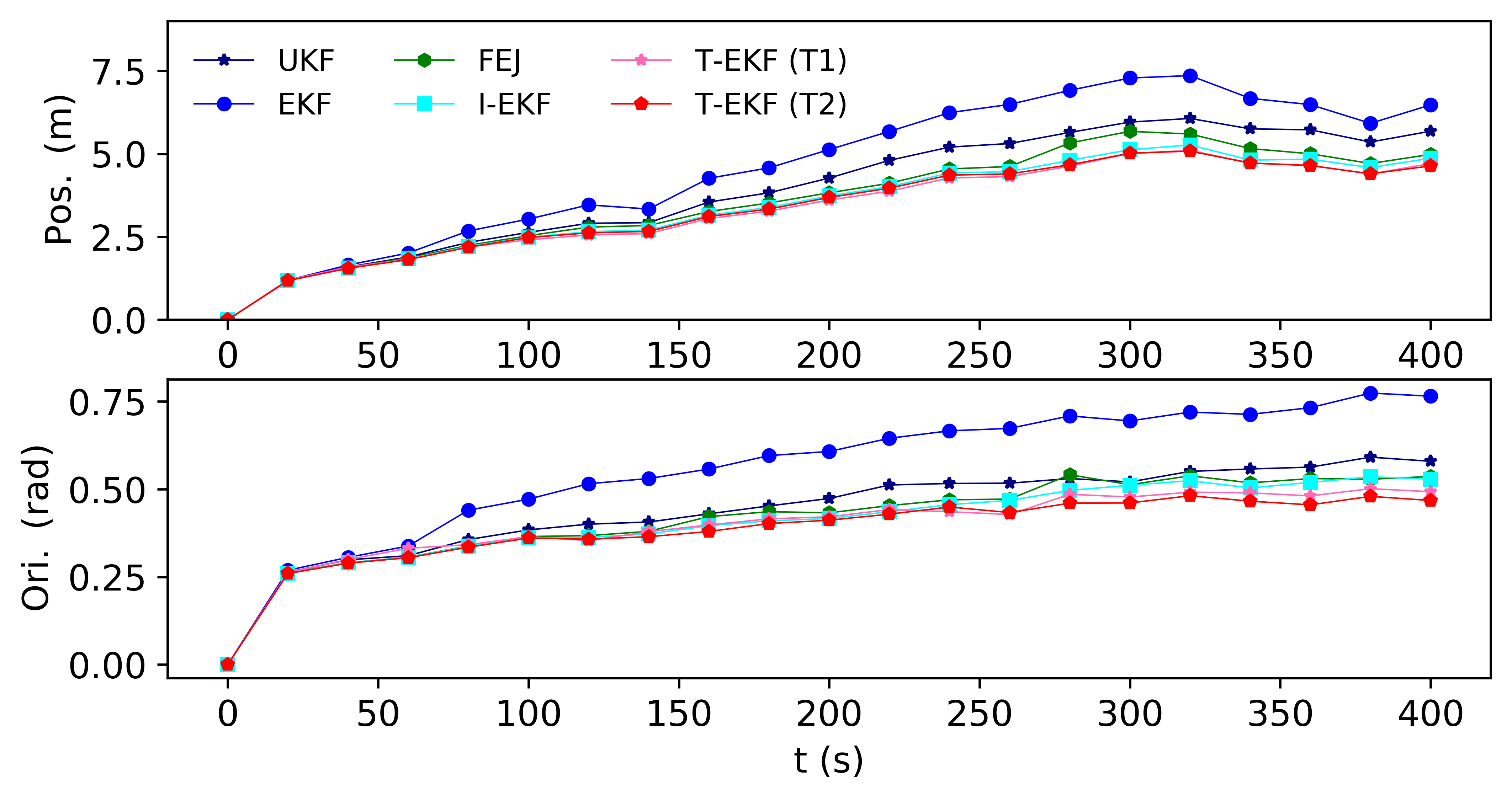}
} \vspace{-8pt} \\
\subfloat[NEES]{
\includegraphics[scale=0.48]{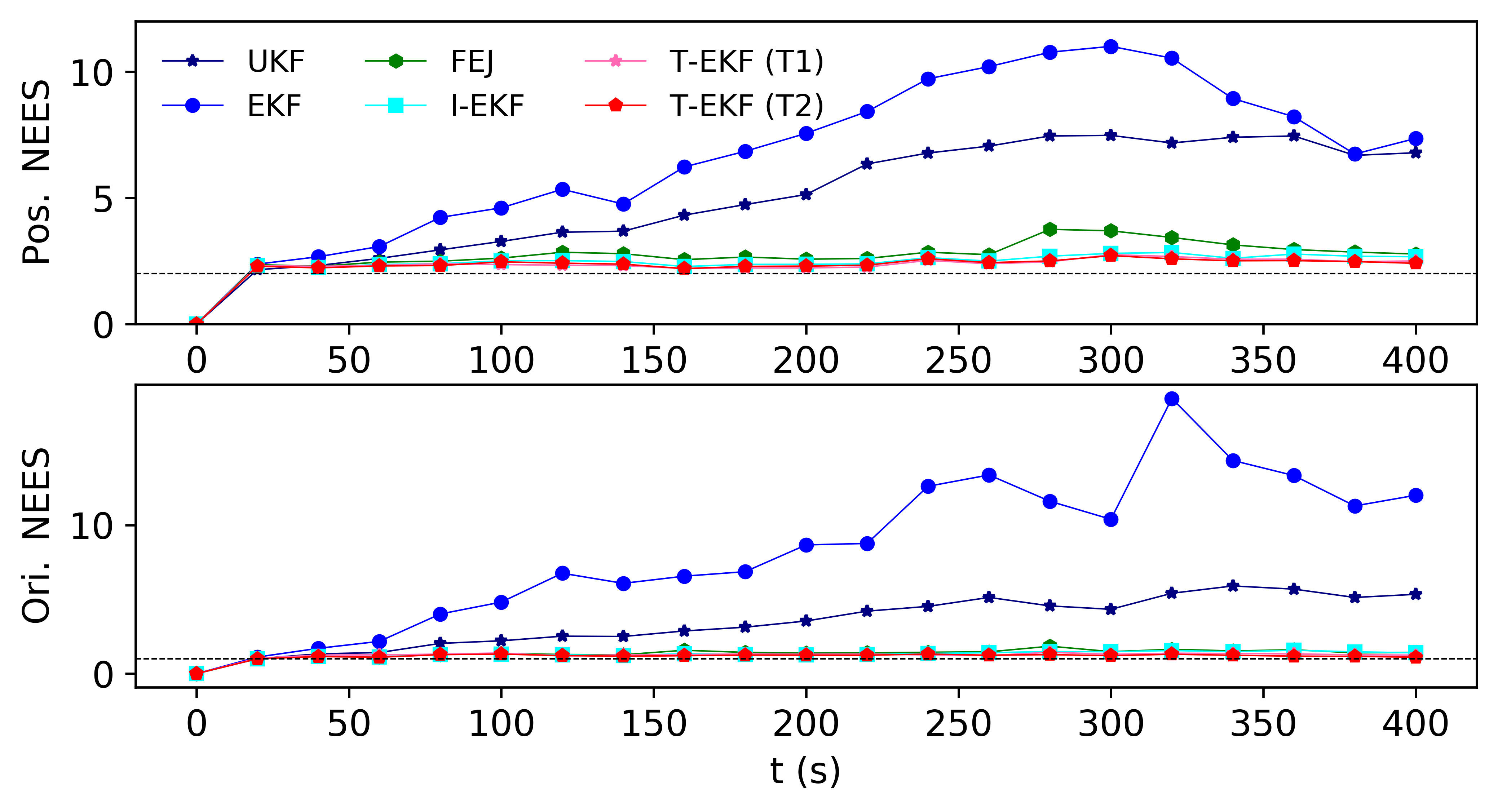}
}
\caption{The plots of statistical RMSE and NEES (position and orientation) of different estimators over Monte Carlo simulation for CL applications.}
\label{fig:cl_nees}
\vspace{-1.0em}
\end{figure}

In the consistency and accuracy test, both implementations of the T-EKF adopt the exact state update approaches. As evident, the proposed T-EKF approach achieves almost consistent estimation results and exhibits better performance compared to other estimators. Amongst, the two estimators, T-EKF 1 and T-EKF 2, yield equivalent results when the same transformation matrix is used. It should be pointed out that the estimation results can vary significantly when different coordinate transformations are applied. In general, selecting the optimal coordinate transformation remains a challenging task, often requiring a process of trial and error.
Fortunately, from the perspective of guaranteeing consistency and enhancing efficiency, we in this paper provide some fundamental principles for the design of coordinate transformations. This will serve as essential guidelines for the development of effective coordinate transformations.

The nonlinear estimator UKF and the classical EKF encounter inconsistency issues stemming from the observability mismatch. Consequently, the performance is significantly inferior to those estimators possessing correct observability properties. The FEJ's performance significantly degrades with time going due to that the adopted linearized system does not strictly follow the first-order Taylor expansion, resulting in larger linearization errors and degradation of accuracy. In contrast, the linearized systems utilized by the I-EKF and the T-EKF not only maintain correct observability properties but also achieve first-order optimality. Consequently, these two estimators behave more accurately and reliably.

\subsubsection{Evaluation of Computational Efficiency}
Additionally, we validate the average running time of these estimators at each single iteration with results presented in Table~\ref{table:cl_sim}. We conduct the simulation test on a desktop consisting of an AMD® Ryzen 9 3900x CPU and 32 GB memory. As expected, the nonlinear estimator UKF requires more computational resources and significantly exceeds the classical EKF and its variants as it necessitates to maintain a lot of sigma points. The computational costs of the proposed T-EKF estimators are slightly higher than that of the classical EKF due to the operations associated with coordinate transformation. Among the two proposed implementation methods, the T-EKF 1 exhibits superior computational efficiency. It is noteworthy that computational requirements might vary when adopting different coordinate transformations. In the context of cooperative localization (CL), the designed second transformation (T2) contributes to better computational performance.

\subsubsection{Evaluation of State Update Approaches}
Furthermore, we conduct simulation experiments to compare the two proposed state update approaches (exact and approximate state update methods). Notably, the two methods exhibit significant differences primarily when the predicted values deviate from the estimates. In actual applications, state predictions may diverge from the estimated values as the measurement detection probability decreases. Therefore, to thoroughly validate their performance, we perform Monte Carlo simulations across different detection probabilities of relative measurements. In these simulation tests, we implement both update approaches within the framework of the T-EKF 1 estimator and the T2 transformation. We present plots of the statistical position and orientation RMSE under varying detection probabilities in Figure \ref{fig:cl_update}.


\begin{figure}[!htp]
\centering
\includegraphics[scale=0.5]{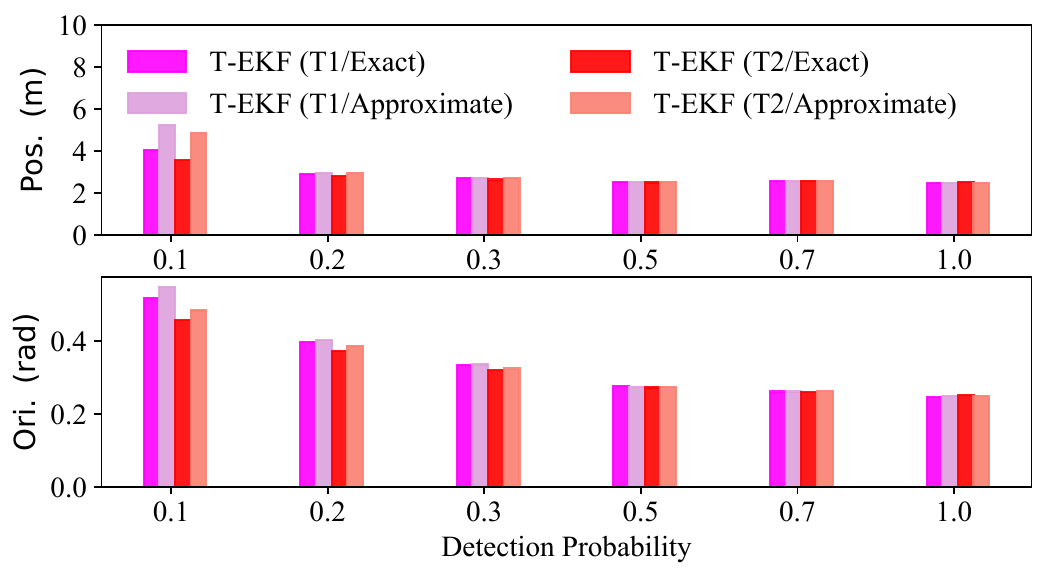}
\caption{The statistical RMSE (position and orientation) under different detection probabilities of relative measurements for CL applications.}
\label{fig:cl_update}
\vspace{-1.0em}
\end{figure}

As demonstrated, the exact state update method achieves better performance across all detection probabilities. In contrast, the approximate state update approach behaves poorly when the detection probability of relative measurements is relatively low. More critically, it may cause the estimator to be divergent under these conditions. As the relative measurement detection probability increases, the results achieved using the approximate state update approach gradually get close to that obtained from the exact state update method. This phenomenon might be attributed to the decreasing discrepancies between the predicted and estimated values as the detection probability of relative measurements rises. In practice, it is better to adopt the exact state update method so as to achieve more accurate and reliable results.

\subsection{Real-world Experiments}
\label{sec:cl_exp}
In this section, we conduct tests on the publicly available UTIAS dataset~\cite{B3}, which was collected using a fleet of five two-wheeled differential drive robots in nine different experiments with different motion profiles. All sub-datasets include a cohesive collection of ego-motion information (linear and angular velocities), relative range-bearing measurements, and ground-truth data. Notably, the ninth one is more challenging since some barriers are intentionally placed in the environment to occlude robots’ views, resulting in a significant reduction in the robots' sensing range and the detection frequency of relative measurements. Additionally, incorrect data associations among relative measurements substantially increase, potentially causing estimators to be divergent. Figure~\ref{fig:cl_traj} illustrates the trajectories of the five robots within the ninth sub-dataset.

\begin{figure}[!htb]
\centering
\includegraphics[scale=0.5, trim=0 0 0 20,clip]{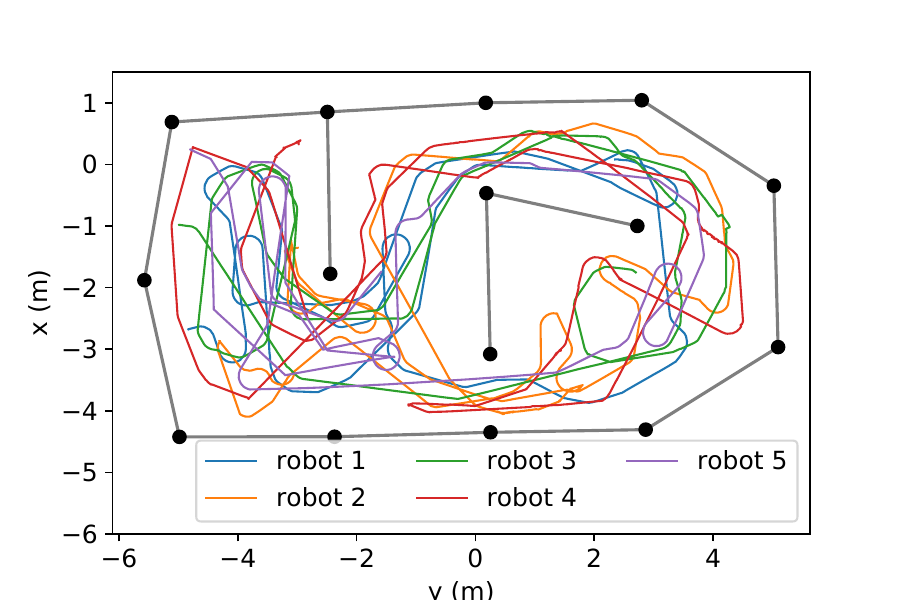}
\caption{The five robots' trajectories in sub-dataset $9$ from $\rm 0$ to $\rm 600s$.}
\label{fig:cl_traj}
\vspace{-1.0em}
\end{figure}

In the dataset experiment, we compare the T-EKF with the nonlinear estimator UKF, and other state-of-the-art filters, including the EKF, the FEJ \cite{B8}, and the I-EKF \cite{B44}. To make a fair comparison, we implement all estimators at the same frequency across all dataset experiments. Additionally, for each estimator, all available relative robot-to-robot measurements are utilized for the estimation updates. The plots of the statistical distribution of position and orientation estimation errors are presented in Figure~\ref{fig:s_error}. Besides, the statistical position and orientation RMSE of these estimators across the nine sub-datasets are summarized in Table~\ref{table:rmse_exp}, with the best results highlighted in bold.

\begin{figure}[!htb]
\centering
\subfloat[Position errors]{
\includegraphics[scale=0.45, trim=0pt 0pt 0pt 30pt, clip]{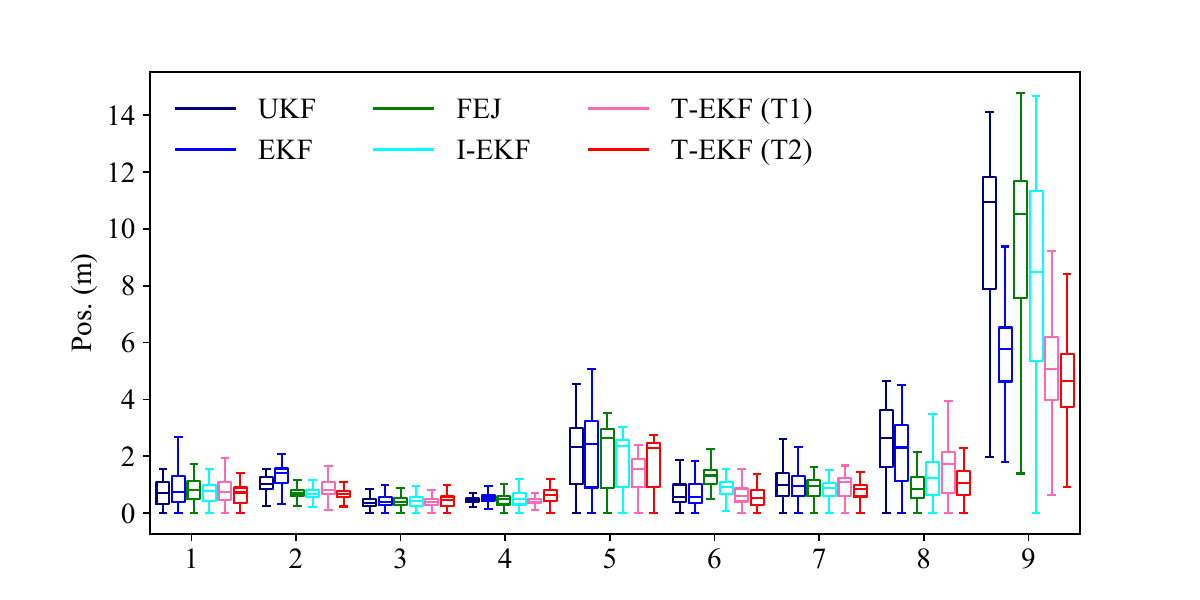}
} \vspace{-8pt} \\
\subfloat[Orientation errors]{
\includegraphics[scale=0.45, trim=0pt 0pt 0pt 30pt, clip]{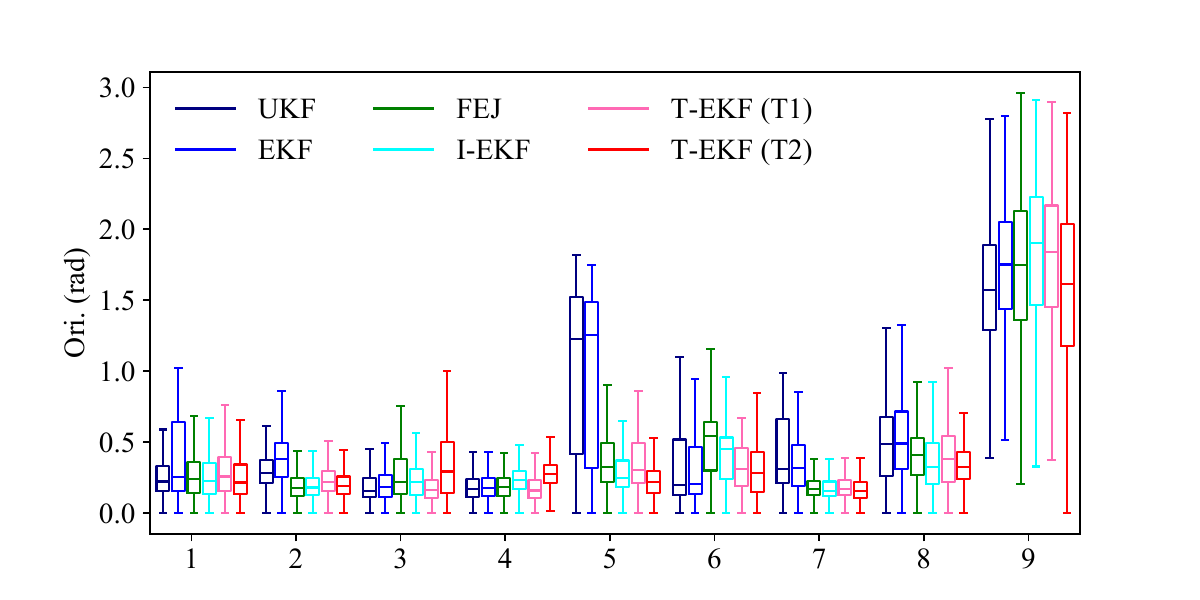}
}
\caption{Statistical error distributions of these estimators over nine sub-datasets for CL applications.}
\vspace{-1.0em}
\label{fig:s_error}
\end{figure}

\begin{table}[!h] 
\centering
\caption{{Average position and orientation RMSE over UTIAS dataset experiments for CL applications.}}
\setlength{\abovecaptionskip}{0cm}
\setlength{\belowcaptionskip}{0cm}
\setlength{\tabcolsep}{1.3pt}
\renewcommand{\arraystretch}{1.5}
\begin{tabular}{|c|c|c|c|c|c|c|} 
\hline
Dataset &  UKF & EKF & FEJ & I-EKF & \makecell{T-EKF \\ (w/ T1)} & \makecell{T-EKF \\ (w/ T2)}  \\
\hline
& \multicolumn{6}{c|}{Position / Orientation RMSE (m / rad)}  \\
\hline
1 & 0.83/ \textbf{0.29} & 1.10/ 0.48 & 0.91/ 0.34 & 0.80/ 0.32 & 0.89/ 0.40 & \textbf{0.72}/ 0.31 \\ \hline
2 & 1.04/ 0.32 & 1.31/ 0.41 & 0.71/ \textbf{0.21} & 0.69/ 0.21 & 0.89/ 0.25 & \textbf{0.67}/ 0.22 \\ \hline
3 & 0.43/ 0.21 & 0.46/ 0.22 & 0.43/ 0.32 & 0.46/ 0.27 & \textbf{0.42}/ \textbf{0.20} & 0.48/ 0.38 \\ \hline
4 & 0.48/ 0.21 & 0.57/ 0.21 & 0.49/ 0.22 & 0.56/ 0.26 & \textbf{0.43}/ \textbf{0.20} & 0.68/ 0.30 \\ \hline
5 & 2.43/ 1.15 & 2.60/ 1.13 & 2.32/ 0.40 & 2.05/ 0.32 & \textbf{1.53}/ 0.40 & 1.96/ \textbf{0.28} \\ \hline
6 & 0.83/ 0.42 & 0.76/ 0.37 & 1.26/ 0.55 & 0.89/ 0.45 & 0.78/ 0.39 & \textbf{0.61}/ \textbf{0.37} \\ \hline
7 & 1.16/ 0.45 & 1.07/ 0.38 & 0.95/ 0.19 & 0.88/ 0.18 & 1.03/ 0.19 & \textbf{0.83}/ \textbf{0.18} \\ \hline
8 & 2.81/ 0.62 & 2.40/ 0.61 & \textbf{1.04}/ 0.49 & 1.50/ 0.46 & 1.74/ 0.49 & 1.17/ \textbf{0.43} \\ \hline
9 & 10.04/ \textbf{1.62} & 5.70/ 1.77 & 9.82/ 1.79 & 8.87/ 1.88 & 5.26/ 1.84 & \textbf{4.87}/ 1.68 \\ \hline
Average & 2.23/ 0.59 & 1.78/ 0.62 & 1.99/ 0.50 & 1.86/ 0.48 & 1.44/ 0.48 & \textbf{1.33}/ \textbf{0.46} \\
\hline
\end{tabular}
\label{table:rmse_exp}
\vspace{-1.0em}
\end{table}

As evident, among these estimators, the T-EKF achieves the best mean performance in both position and orientation estimation. Besides, the T-EKF adopting the second coordinate transformation (T2) achieves better mean performance than that using the first coordinate transformation (T1). The nonlinear estimator UKF and the classical EKF behave worse than the estimators having correct observability properties due to the inconsistency issues arising from observability mismatch. The FEJ and the I-EKF suffer from severe accuracy degradation in sub-dataset $9$. This is because of the influence of measurement outliers resulting from erroneous associations. In comparison, the T-EKF demonstrates better robustness against measurement perturbations. 

Notably, although the T-EKF addresses the inconsistency issue and achieves near-optimal estimation results, it does not consistently demonstrate the highest accuracy across all sub-datasets. This is because the optimality of the estimator is guaranteed only in a statistical sense. Due to inherent randomness, it cannot be assured to be the best in every single test. Furthermore, it can be found that the results obtained using different transformations vary and may even be complementary. This suggests the potential for fusing estimation results from various transformations to achieve improved performance, which will be explored in future research.

\section{Application in Multi-Source Target Tracking}
\label{sec:application_tt}
In this section, we present the application of the proposed methods to a class of observable nonlinear systems, where multiple measurement sources exist with each capable of perceiving only partial states of the system. Although the entire system is observable, the system would become unobservable if each measurement source operates independently. Consequently, the classical EKF may still encounter inconsistency issues in such systems. 

\subsection{Problem Statement of Target Tracking}

Consider a typical scenario of multi-source target tracking wherein a robot moves in a two-dimensional plane and is tracked by observing multiple landmarks on the plane. These landmarks are scattered in the environment and provide relative bearing measurements to the target alternatively. By using the intermittent bearing measurements from multiple landmarks as well as the ego-motion information, we can estimate the target’s pose with respect to the global reference frame. 

Consider an omnidirectional robot platform with its motion model given by
\begin{align} \label{equ:tt_system}
\mathbf{p}_{k+1} &= \mathbf{p}_{k} + \mathbf{R}({\psi_{k}}) \Big(\mathbf{v}_{k} + \boldsymbol{\nu}_{i, k} \Big) \delta t, \\
\psi_{k+1} &= \psi_{k} + \Big(\omega_{k} + \varpi_{i, k} \Big) \delta t, 
\label{equ:tt_system_}
\end{align}
where $k$ denotes the discrete-time index; $\mathbf{p}_{k} \in \mathbb{R}^2$ and $\psi_{k} \in \mathbb{R}$ are the robot's position and orientation at time step $k$, respectively;
$\mathbf{v}_{k} \in \mathbb{R}^2$ and $\omega_{k} \in \mathbb{R}$ are the robot's linear and angular velocities; $\boldsymbol{\nu}_{i, k}$ and $\varpi_{i, k}$ denote the zero-mean white Gaussian input noises; $\delta t$ is the sampling period, and $\mathbf{R}(\cdot) \in \mathbb{R}^{2\times 2}$ represents the rotation matrix. 

The bearing measurement provided by the $i$-th landmark at time step $k$ is expressed as:
\begin{equation}
z_{i, k} 
=
\tan^{-1} \left( \dfrac{y_{{\rm s}_i} - y_{k}}{x_{{\rm s}_i} - x_{k}} \right) - \psi_{k} + \eta_{i, k}
\label{equ:tt_relative_bearing}
\end{equation} 
where $\mathbf{p}_{k} = [x_{k} \: y_{k}]^{\top}$ denotes the robot's position, $\mathbf{p}_{{\rm s}_i} = [x_{{\rm s}_i} \: y_{{\rm s}_i}]^{\top}$ represents the known position of the $i$-th landmark, and $\eta_{i, k} \in \mathbb{R}$ denotes the measurement noise subject to a zero-mean Gaussian distribution. 

\subsection{Transformation Design}
\label{sec:application_tt_trans}

Note that the system described in \eqref{equ:tt_system}-\eqref{equ:tt_relative_bearing} is observable if and only if at least two bearing measurements can be accessed. Otherwise, the system is not observable. To mitigate the potential inconsistency issue caused by intermittent bearing measurements, we first analyze the system's unobservable subspace under the condition of one single landmark and then present the design of coordinate transformations.

Suppose only the $i$-th landmark is available. Then the system is not observable with the unobservable subspace as follows
\begin{equation}
\mathbb{N}_{i, k}
=
\Span_{\col}
\left( 
\left[
\begin{array}{c}
\mathbf{J} (\mathbf{p}_{k} - \mathbf{p}_{{\rm s}_{i}}) \\ 
1 \\
\end{array}
\right]
\right) .
\label{equ:tt_N_x}
\end{equation}
To mitigate the inconsistency issue, we aim to find a common transformation matrix such that the observable subspaces associated with different landmarks can be changed as state-independent simultaneously. To this end, we construct a transformation matrix as follows
$$
\mathbf{T}_{k}
=
\left[
\begin{array}{cc}
\mathbf{I}_2 & - \mathbf{J} \mathbf{p}_{k} \\
\mathbf{0}_{1 \times 2} & 1
\end{array}
\right]
$$
by which the state propagation Jacobian matrix becomes an identity matrix. Then, in light of Theorem \ref{theorem:trans_design_2}, the system's unobservable subspace also becomes constant.
Thus, the inconsistency issue caused by intermittent bearing measurements can be mitigated.

\subsection{Simulation Experiments}

To validate the proposed method, we conduct $200$ Monte Carlo simulation trials to evaluate the estimators’ performance. 
In this simulation test, the target robot is controlled to move along a circular trajectory at a constant velocity of $\rm 0.3 \: m/s$ and $\rm 0.1 \: rad/s$. The standard deviations of linear and angular velocity measurement noises are set to be ${\rm 0.15 \: m/s}$ and ${\rm 0.06 \: rad/s}$, respectively. Two landmarks with known positions are adopted to alternately yield relative bearing measurements to the target. The standard deviation of bearing measurement noises is $\rm 0.1 \: rad$. The simulation step is set to ${\rm 0.4 \: s}$. It should be noted that the noise levels and the simulation step are purposely set to be larger than that encountered in practice so as to make the effects of inconsistency more apparent.

\begin{table}[h] 
\centering
\caption{Average RMSE, NEES, and running time over Monte Carlo simulations for target tracking applications.}
\setlength{\abovecaptionskip}{0cm}
\setlength{\belowcaptionskip}{0cm}
\setlength{\tabcolsep}{6.8pt}
\renewcommand{\arraystretch}{1.3}
\begin{tabular}{|c|c|c|c|c|c|} 
\hline
Estimators & \makecell{RMSE \\ Pos. (m)} & \makecell{RMSE \\ Ori. (rad)} & \makecell{NEES \\ Pos.} & \makecell{NEES \\ Ori.} & \makecell{{ Running} \\ {Time (ms)}} \\
\hline
DR & 3.250 & 0.608 & \textbf{2.120} & \textbf{1.060} & {0.068} \\
\hline
EKF & 1.587 & 0.316 & 5.499 & 2.809 & {0.192} \\
\hline
OC-EKF  & 1.351 & 0.257 & 3.541 & 1.094 & 0.199 \\
\hline
I-EKF & \color{blue}{1.348} & \color{blue}{0.256} & 3.923 & 1.087 & 0.229 \\
\hline
T-EKF 1 & \multirow{2}{*}{\textbf{1.342}} & \multirow{2}{*}{\textbf{0.254}} & \multirow{2}{*}{\color{blue}{3.287}} & \multirow{2}{*}{\color{blue}{1.070}} & 0.223 \\
\cline{1-1} \cline{6-6}
T-EKF 2 & & & & & 0.271 \\
\hline
\end{tabular}
\label{table:tt_sim}
\vspace{-1.0em}
\end{table}

We compare the performance of the T-EKF with the dead reckoning (DR) method as well as state-of-the-art estimation methods, including the classical EKF, the Observability-Constrained (OC) EKF \cite{B8}, and the Invariant EKF (I-EKF) \cite{B44}. The RMSE and NEES of these estimators are summarized in Table~\ref{table:tt_sim}. Amongst, the best results are highlighted in bold. Besides, we draw the plots of the average RMSE and NEES in Figure \ref{fig:tt_nees}. 

\begin{figure}[!htp]
\centering
\subfloat[RMSE]{
\includegraphics[scale=0.48]{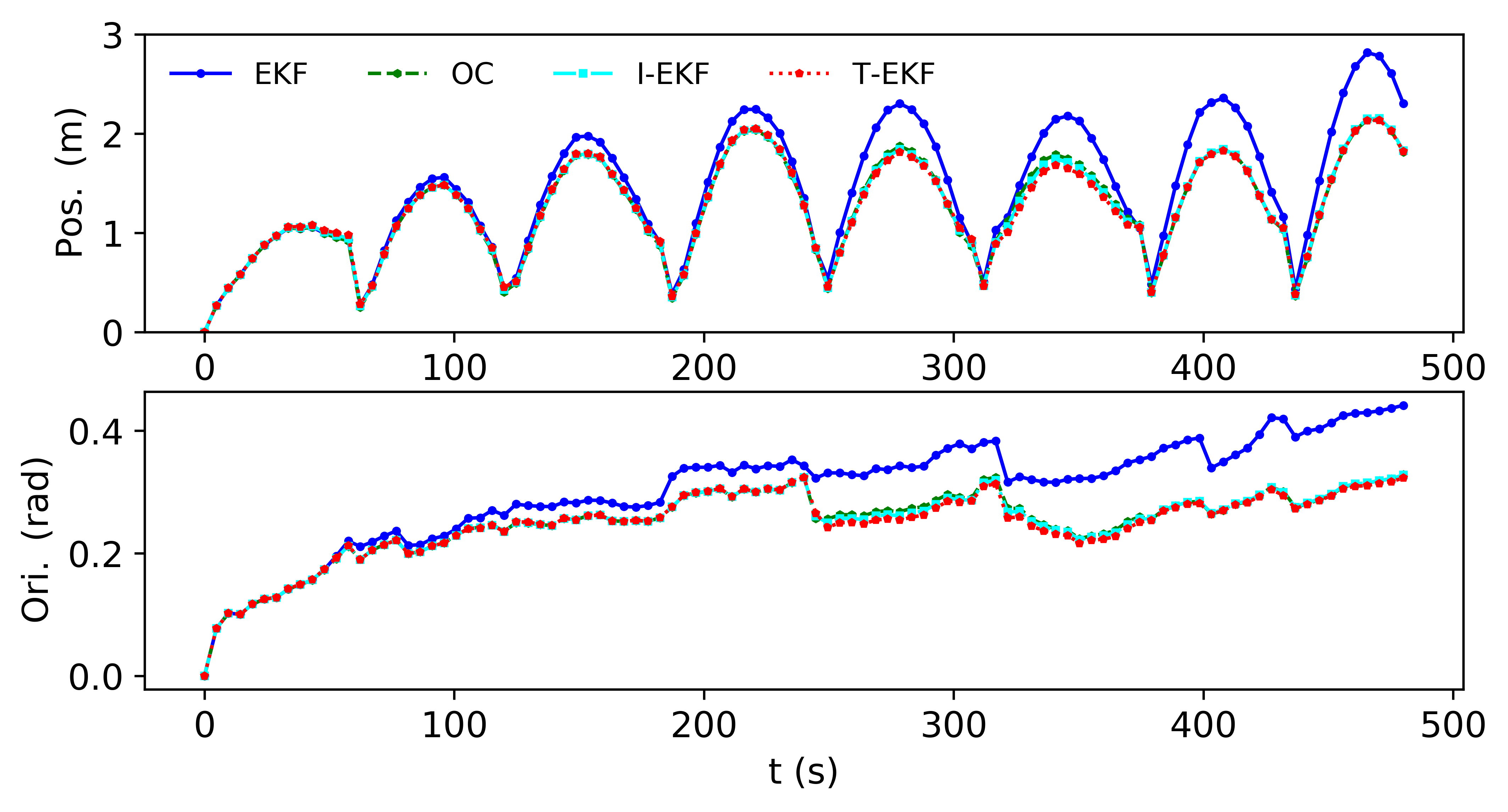}
} \vspace{-8pt} \\
\subfloat[NEES]{
\includegraphics[scale=0.48]{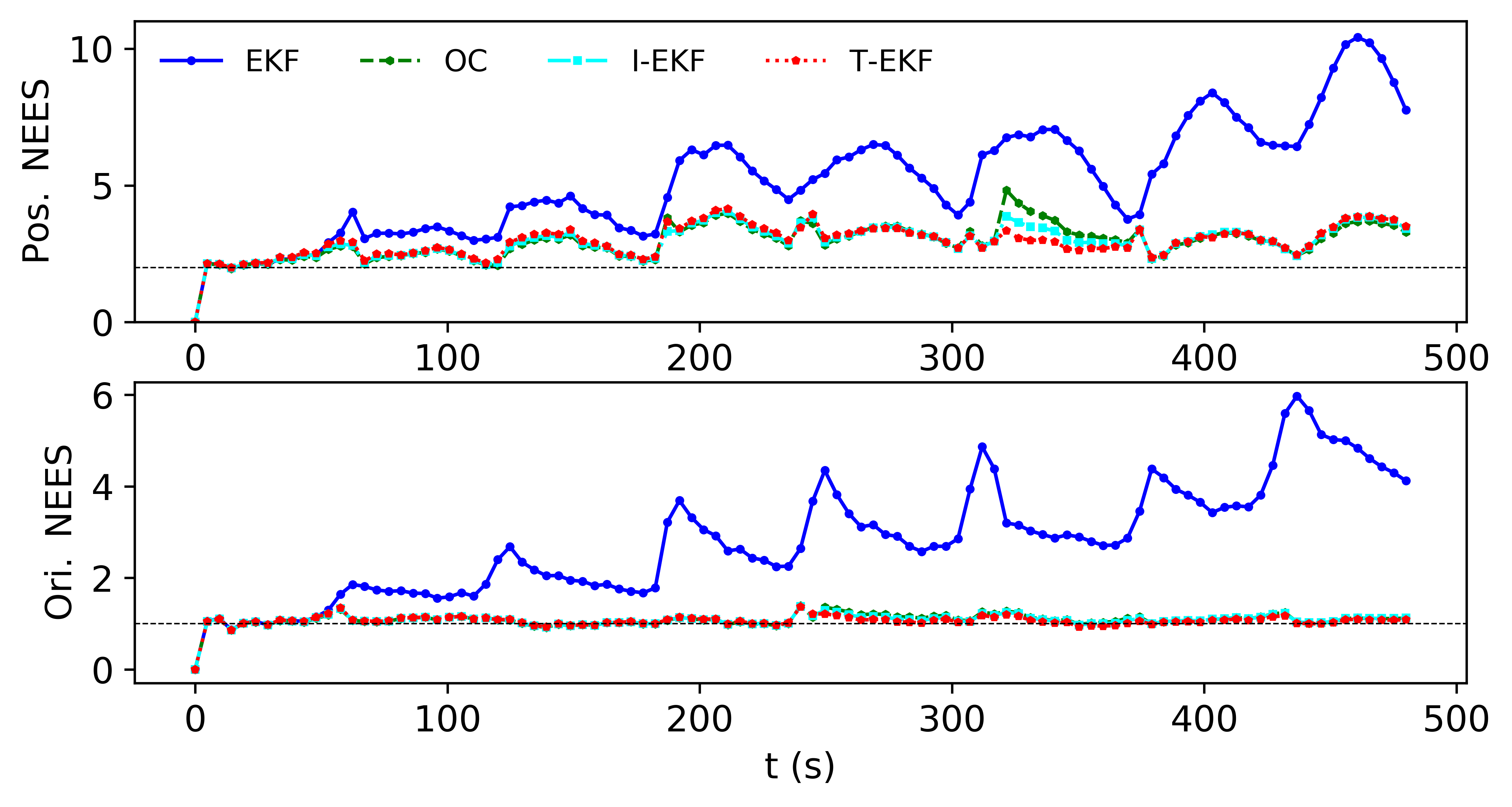}
}
\caption{The plots of statistical RMSE and NEES (position and orientation) over 200 Monte Carlo simulation trials for target tracking application.}
\label{fig:tt_nees}
\vspace{-1.0em}
\end{figure}

As evident, the classical EKF is inconsistent and significantly inferior to those estimators possessing correct observability properties. This is due to the fact that the estimator mistakes the system's unobservable directions and acquires spurious information when the landmarks provide observations individually. Notably, the inconsistency effect would be relieved as the interval between alternating landmark observations decreases. In comparison, benefiting from guaranteeing correct observability properties, the OC-EKF, the I-EKF, and the T-EKF achieve almost consistent estimation results and exhibit better performance. In addition, the I-EKF and the T-EKF behave more accurately and reliably since the first-order optimality is guaranteed.

\subsection{Real-world Experiments}
We further conduct tests on the publicly available UTIAS dataset~\cite{B3}. Different from the experiment presented in Section \ref{sec:cl_exp} wherein all landmarks are discarded, we regard the landmarks as known anchors and utilize the bearing measurements from these landmarks to track the robots' pose. To verify the performance of the proposed method, we compare it with state-of-the-art baselines, including the classical EKF, the OC-EKF, and the I-EKF. The position and orientation RMSE of these estimators across the nine sub-datasets are summarized in Table~\ref{table:rmse_exp_tt}, with the best results highlighted in bold and the second best results printed in blue.

\begin{table}[!htbp] 
\centering
\caption{{Average position and orientation RMSE over UTIAS dataset experiments for target tracking applications.}}
\setlength{\abovecaptionskip}{0cm}
\setlength{\belowcaptionskip}{0cm}
\setlength{\tabcolsep}{1.5pt}
\renewcommand{\arraystretch}{1.5}
\begin{tabular}{|c|c|c|c|c|} 
\hline
Dataset & EKF & OC-EKF & I-EKF & T-EKF  \\
\hline
& \multicolumn{4}{c|}{Position / Orientation RMSE (m / rad)}  \\
\hline
1 & {\color{blue} 0.1954} / {\color{blue} 0.1125} & \textbf{0.1926} / \textbf{0.1121} & 0.2044 / 0.1162 & 0.2048 / 0.1158 \\ \hline
2 & 0.2146 / 0.1186 & 0.2147 / \textbf{0.1181} & {\color{blue} 0.2117} / 0.1183 & \textbf{0.2083} / {\color{blue} 0.1182} \\ \hline
3 & 0.1384 / 0.1153 & 0.1353 / 0.1147 & {\color{blue} 0.1321} / {\color{blue} 0.1141} & \textbf{0.1300} / \textbf{0.1136} \\ \hline
4 & 0.1461 / 0.0828 & 0.1442 / 0.0827 & {\color{blue} 0.1417} / {\color{blue} 0.0824} & \textbf{0.1388} / \textbf{0.0823} \\ \hline
5 & 0.4216 / 0.1889 & 0.4102 / 0.1878 & {\color{blue} 0.4030} / {\color{blue} 0.1862} & \textbf{0.4009} / \textbf{0.1852} \\ \hline
6 & 0.2548 / 0.1363 & \textbf{0.2408} / 0.1336 & {\color{blue} 0.2428} / {\color{blue} 0.1334} & 0.2433 / \textbf{0.1329} \\ \hline
7 & 0.3382 / 0.1670 & 0.3278 / 0.1664 & {\color{blue} 0.3277} / {\color{blue} 0.1655} & \textbf{0.3245} / \textbf{0.1651} \\ \hline
8 & 0.3408 / 0.2202 & 0.3376 / 0.2190 & \textbf{0.3309} / {\color{blue} 0.2178} & {\color{blue} 0.3362} / \textbf{0.2175} \\ \hline
9 & 0.2779 / 0.5022 & \textbf{0.2714} / 0.4975 & {\color{blue} 0.2727} / {\color{blue} 0.4958} & 0.2792 / \textbf{0.4944} \\ \hline
Average & 0.2586 / 0.1826 & 0.2527 / 0.1813 & {\color{blue} 0.2518} / {\color{blue} 0.1810} & \textbf{0.2517} / \textbf{0.1805}  \\
\hline
\end{tabular}
\label{table:rmse_exp_tt}
\vspace{-1.0em}
\end{table}

As seen, these estimators behave comparably across the nine sub-datasets. This is attributed to the robot's frequent access to bearing measurements from multiple landmarks, which makes the system remain observable most of the time. Consequently, the inconsistency phenomenon caused by the erroneous perception of observability properties is not as obvious as that observed in simulations. Nevertheless, the classical EKF continues to exhibit inconsistency resulting from observability mismatch and behaves worse than the other estimators. Among these, the T-EKF achieves the best mean performance in both position and orientation estimation and outperforms others in most of the nine sub-datasets. 

\section{Application in Visual-Inertial Navigation}
\label{sec:application_vins}

In this section, we present the application of the proposed method to address the inherent inconsistency problem in 3D visual-inertial navigation systems. 

\subsection{Problem Statement of VINS}
\label{sec:vio}
VINS, with an IMU and a camera, aims to simultaneously estimate the IMU state $\mathbf{x}_{I}$ and the unknown feature state $\mathbf{x}_{f}$:
\begin{equation}
\begin{aligned}
\mathbf{x}
&=
\big( \mathbf{x}_{I}, \ \mathbf{x}_{f} \big) \\
&=
\big( {^I_G \mathbf{R}}, \ {^G \mathbf{p}_{I}}, 
\ {^G \mathbf{v}_{I}}, \ \mathbf{b}_{g}, \ \mathbf{b}_{a}, \ {^{G}\mathbf{p}_{f_1}}, \ \cdots, \ {^{G}\mathbf{p}_{f_m}} \big)
\end{aligned}
\end{equation}
where ${^I_G \mathbf{R}} \in \mathbf{SO}(3)$ represents the rotation from the global frame $\{G\}$ to the IMU frame $\{I\}$; ${^G \mathbf{p}_{I}} \in \mathbb{R}^3$ and ${^G \mathbf{v}_{I}} \in \mathbb{R}^3$ represent the IMU position and velocity in the global frame $\{\mathcal{G}\}$, respectively; $\mathbf{b}_g \in \mathbb{R}^3$ and $\mathbf{b}_a \in \mathbb{R}^3$ denote the biases of the gyroscope and accelerometer, respectively; ${^{G}\mathbf{p}_{f_i}} \in \mathbb{R}^3$ denotes the 3D position of the $i$-th unknown feature ($i =1, ..., m$). 

The discrete-time system model \cite{B71} that describes the time evolution of the IMU state can be written as
\begin{equation} \label{equ:imu_evol}
\mathbf{x}_{I_{k+1}} = \mathbf{f}(\mathbf{x}_{I_{k}}, \boldsymbol{\omega}_{m_k}, \boldsymbol{a}_{m_k}, \mathbf{n}_k)
\end{equation}
where $\mathbf{n}_k$ consists of the zero-mean white Gaussian noise of the IMU measurements along with random walk bias noises; $\boldsymbol{\omega}_{m_k}$ and $\mathbf{a}_{m_k}$ denote the gyroscope and accelerometer measurements, which are modeled as
\begin{align}
\boldsymbol{\omega}_{m_k} &= {\boldsymbol{\omega}_{I_k}} + \mathbf{b}_{g_k} + \mathbf{n}_{g_k} \\
\boldsymbol{a}_{m_k} &= {\boldsymbol{a}_{I_k}} - {^{I_k}_G \mathbf{R}} {^G \mathbf{g}} + \mathbf{b}_{a_k} + \mathbf{n}_{a_k}
\end{align}
where ${\boldsymbol{\omega}_{I_k}}$ and ${\boldsymbol{a}_{I_k}}$ represent the angular velocity and the linear acceleration of the system, respectively; ${^G \mathbf{g}}$ denotes the gravitational acceleration expressed in the global frame $\{G\}$; $\mathbf{n}_{g_k}$ and $\mathbf{n}_{a_k}$ denote zero-mean white Gaussian noises of the gyroscope and accelerometer, respectively.

The visual measurement model can be written as the following series of nested functions:
\begin{equation} \label{equ:vis_m}
\begin{aligned}
\mathbf{z}_{i_k} 
&= \mathbf{h}(\mathbf{x}_{I_{k}}, {^{G}\mathbf{p}_{f_i}}) + \mathbf{w}_{i_k} \\
&= \mathbf{h}_{p}({^{C_k}\mathbf{p}_{f_i}}) + \mathbf{w}_{i_k} \\
&= \mathbf{h}_{p}(\mathbf{h}_{t}({^{G}\mathbf{p}_{f_i}}, {^{I_k}_G \mathbf{R}}, {^G\mathbf{p}_{I_{k}}}, {^{C}_{I}\mathbf{R}}, {^{C}\mathbf{p}_{I}})) + \mathbf{w}_{i_k} 
\end{aligned}
\end{equation}
where $\mathbf{z}_{i_k}$ is the normalized undistorted visual measurement; $\mathbf{w}_{i_k}$ is the zero-mean white Gaussian measurement noise; ${^{C_k}\mathbf{p}_{f_i}}$ denotes the feature position in the current camera frame $\{C_{k}\}$; ${^{G}\mathbf{p}_{f_i}}$ represents the feature position in the global frame $\{G\}$;  ${^{C}_{I}\mathbf{R}}$ and ${^C\mathbf{p}_{I}}$ denote the system's extrinsic parameters;
$\mathbf{h}_p(\cdot)$ represents the projection function of the camera
\begin{equation} \label{equ:prop_func}
\begin{aligned}
\mathbf{z}_{i_k} 
\triangleq
\mathbf{h}_p({^{C_k}\mathbf{p}_{f_i}}) 
= 
\left[
\begin{array}{c}
\dfrac{^{C_k}x_{f_i}}{^{C_k}z_{f_i}} \\
\dfrac{^{C_k}y_{f_i}}{^{C_k}z_{f_i}} \\
\end{array}
\right]
\end{aligned}
\end{equation}
with 
$
{^{C_k}\mathbf{p}_{f_i}} = [^{C_k}x_{f_i} \ ^{C_k}y_{f_i} \ ^{C_k}z_{f_i}]^{\top};
$
and $\mathbf{h}_t(\cdot)$ represents the coordinate transformation function expressed as
\begin{equation} \label{equ:trans_func}
\begin{aligned}
{^{C_k}\mathbf{p}_{f_i}} 
&\triangleq 
\mathbf{h}_{t}(\cdot) 
=\mathbf{h}_{t}({^{G}\mathbf{p}_{f_i}}, {^{I_k}_G \mathbf{R}}, {^G \mathbf{p}_{I_{k}}}, {^C_I\mathbf{R}}, {^C \mathbf{p}_{I}}) \\
\end{aligned}
\end{equation}

\subsection{Transformation Design}
\label{sec:application_vins_trans}
 
In the following, we will present the design of coordinate transformations to mitigate the inconsistency issue in VINS.
To this end, we first give the system's unobservable subspace. According to \cite{B34}, the system's unobservable subspace spans along the direction of the global position and the global orientation around gravity as follows
\begin{equation}
\mathbb{N}_{k}
= \Span_{\col}
\left( 
\left[
\begin{array}{cc}
\mathbf{0}_{3 \times 3} &  {^G \mathbf{g}} \\
\mathbf{I}_{3} & -[{^G \mathbf{p}_{I_k}}]_\times {^G \mathbf{g}} \\
\mathbf{0}_{3 \times 3} & -[{^G \mathbf{v}_{I_k}}]_\times {^G \mathbf{g}} \\
\mathbf{0}_{6 \times 3} & \mathbf{0}_{6\times 1} \\
\mathbf{I}_{3} & - [{^G \mathbf{p}_{f_1}}]_\times {^G \mathbf{g}} \\
\vdots &\vdots \\
\mathbf{I}_{3} & - [{^G \mathbf{p}_{f_m}}]_\times {^G \mathbf{g}} \\
\end{array}
\right]
\right) .
\label{equ:VINS_N_G}
\end{equation}
Observing the state-dependent unobservable direction, i.e., the last column in \eqref{equ:VINS_N_G}, according to Theorem \ref{theorem:trans_design_1}, we can design the following transformation matrix:
\begin{equation} \label{equ:T_for_VINS_global}
\mathbf{T}_{k} = 
\left[
\begin{array}{c|cccccc}
\mathbf{I}_3 & \multicolumn{6}{c}{\raisebox{0ex}[0pt]{\large 0}} \\ \hline
- \lbrack {^G \mathbf{p}_{I_k}} \rbrack_{\times} &  \\
- \lbrack {^G \mathbf{v}_{I_k}} \rbrack_{\times} &  \\
\mathbf{0}_{6\times 3} & \\
- \lbrack {^G \mathbf{p}_{f_1,k}} \rbrack_{\times} &  \\
\vdots & \\
- \lbrack {^G \mathbf{p}_{f_m,k}} \rbrack_{\times} & \multicolumn{6}{c}{\raisebox{8ex}[0pt]{\LARGE I}}  \\
\end{array}
\right]^{-1} .
\end{equation}
%
%
In light of Lemma \ref{lemma1}, applying the above transformation matrix results in the following constant (state-independent) unobservable subspace: 
\begin{equation}
\bar{\mathbb{N}}_{k}
= \Span_{\col}
\left( 
\left[
\begin{array}{cc}
\mathbf{0}_{3 \times 3} &  {^G \mathbf{g}} \\
\mathbf{I}_{3} & \mathbf{0}_{3\times 1} \\
\mathbf{0}_{3 \times 3} & \mathbf{0}_{3\times 1} \\
\mathbf{0}_{6 \times 3} & \mathbf{0}_{6\times 1} \\
\mathbf{I}_{3} & \mathbf{0}_{3\times 1} \\
\vdots &  \vdots \\
\mathbf{I}_{3} & \mathbf{0}_{3\times 1} \\
\end{array}
\right]
\right).
\label{equ:VINS_N_G_T}
\end{equation}
Evidently, the state-dependent unobservable direction becomes constant, while the state-independent unobservable directions remain invariant. Notably, in the transformation design process described above, we focus solely on the state-dependent unobservable direction and design transformations in accordance with Theorem \ref{theorem:trans_design_1} to treat them. The state-independent unobservable directions are automatically preserved.

\subsection{Simulation Experiments}

To verify the effectiveness of the proposed method in addressing the inconsistency issue in VINS, we conduct Monte Carlo simulations and compare it with state-of-the-art approaches, including the classical EKF, the FEJ \cite{B16}, and the I-EKF \cite{B31}. To eliminate the potential influence of the estimation framework on algorithm evaluation, we select OpenVINS as a unified testing platform to evaluate these estimators. It should be noted that the EKF and the FEJ are natively supported by OpenVINS, while the I-EKF and the T-EKF are developed by ourselves.

\begin{figure}[htbp]
\centering
\includegraphics[scale = 0.1, trim=0pt 20pt 0pt 150pt, clip]{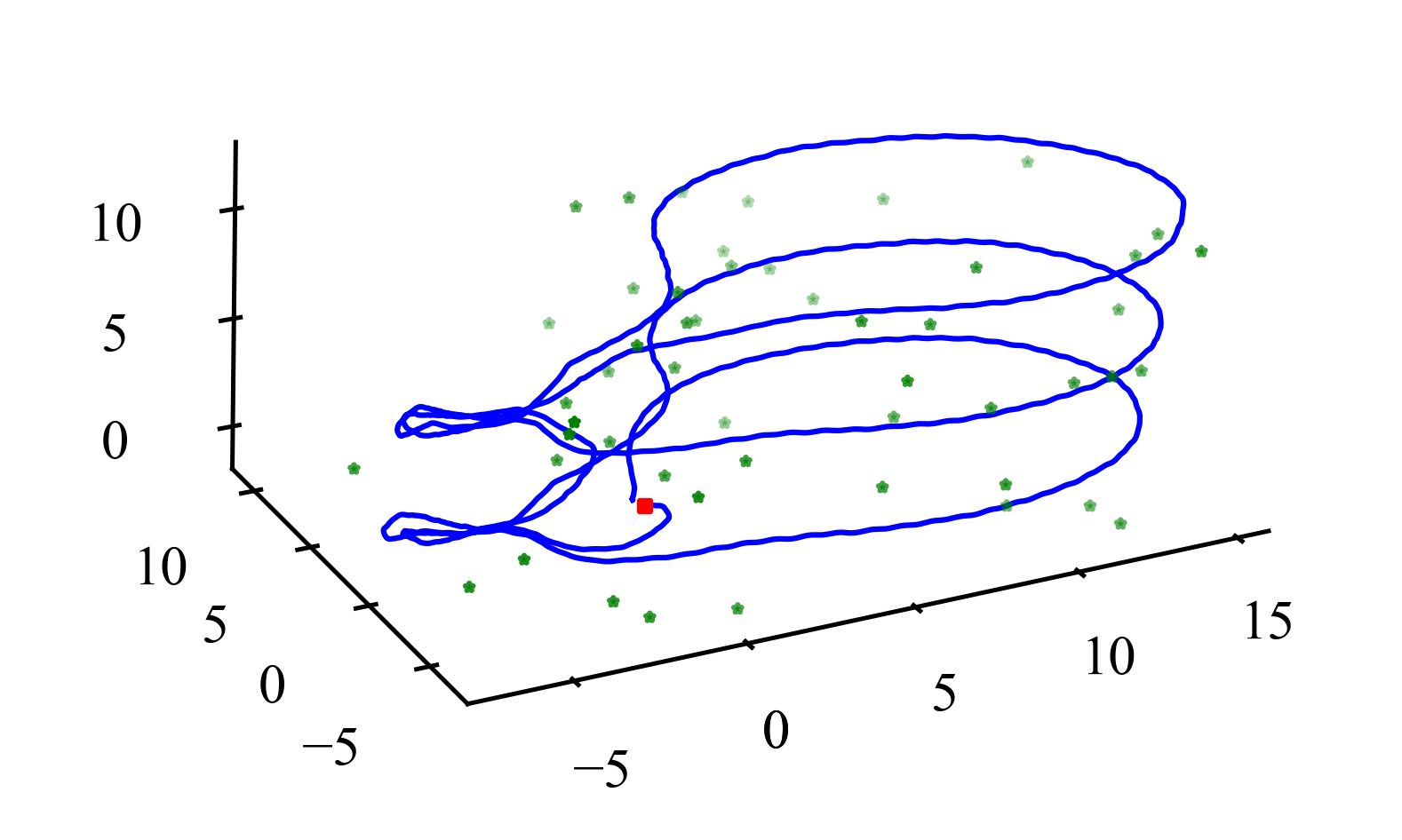}
\caption{ The motion trajectory in Monte Carlo simulations. The red square is the start of the trajectory and the green stars are the features.
}
\label{fig:slam_gore_traj}
\vspace{-1.0em}
\end{figure}

We employ OpenVINS's built-in simulator to generate measurement data and select the Udel-gore as the motion trajectory~\cite{B18,B33}, as shown in Figure \ref{fig:slam_gore_traj}. To fully validate the performance of these estimators under different noise conditions, we perform multiple Monte Carlo simulations with the visual measurement noise set to be $1$, $3$, and $5$ pixels, respectively. In each simulation, every estimator undergoes 1000 trials. For clarity, we summarize the main simulation parameters in Table \ref{tab:VINS_Simu_Param}. 

\begin{table}[htbp]
\centering
\caption{Parameter setup in Monte Carlo simulations.}
\setlength{\abovecaptionskip}{0cm}
\setlength{\belowcaptionskip}{0cm}
\setlength{\tabcolsep}{8pt}
\renewcommand{\arraystretch}{1.5}
\begin{tabular}{cc|cc}
\hline
\textbf{Parameters}   &  \textbf{Value} & \textbf{Parameters}   &  \textbf{Value} \\
\hline
Gyro. Walk Noise & 1.9393e-5 & Acc. Walk Noise & 3.0000e-3 \\
Gyro. Meas. Noise & 1.6968e-4 & Acc. Meas. Noise & 2.0000e-3 \\
IMU Freq. & 200 & Cam. Freq. & 10\\ 
Cam. Meas. Noise & [1, 3, 5] & Max SLAM Point & 40 \\ 
Cam. Num. & Mono & Max Cloned IMU &11 \\
\hline
\end{tabular}
\label{tab:VINS_Simu_Param}
\vspace{-1.0em}
\end{table}

\subsubsection{Accuracy and Consistency Evaluation}

As with the preceding simulation, we also employ the root mean square error (RMSE) and normalized estimation error squared (NEES) to quantify the accuracy and consistency of these estimators. The average position and orientation RMSE values for these estimators under varying levels of visual measurement noise are summarized in Table \ref{table:rmse_exp}, with the best results highlighted in bold. Additionally, Figure \ref{fig:vins_rmse_overtime} illustrates the position and orientation RMSE over time for different visual measurement noise levels. Furthermore, Figure \ref{fig:VINS_NEES} displays the statistical distributions of the position and orientation NEES for these estimators under varying noise conditions. It is important to note that within the 3D visual-inertial navigation application, both the position and orientation NEES should approach $3$ for an estimator to be considered consistent.

\begin{figure}[htbp]
\centering
\includegraphics[width=0.99\linewidth]{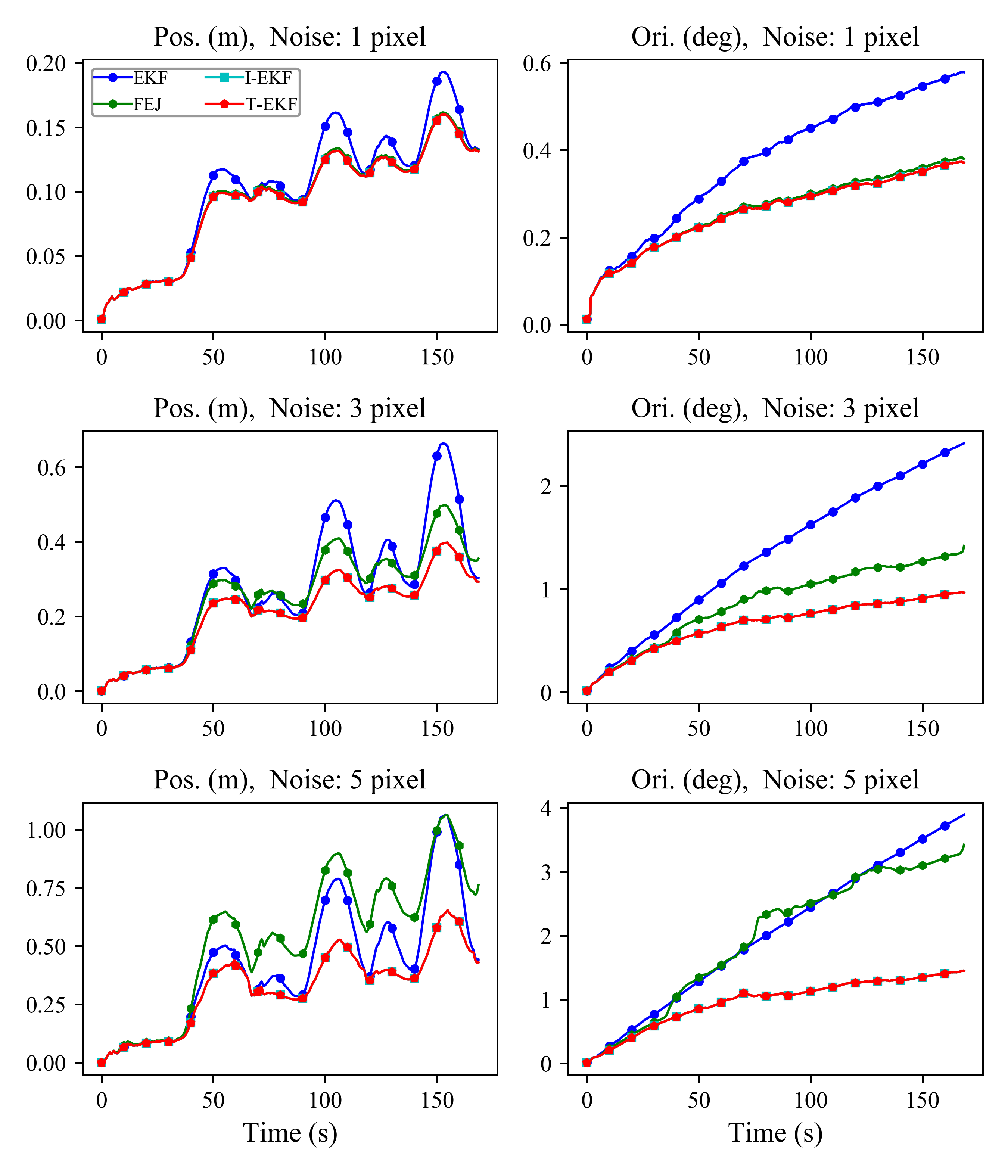}
\caption{The plots of position and orientation RMSE in Monte Carlo simulations under different measurement noise levels. Notably, the plots of the I-EKF and the T-EKF almost coincide and are nearly indistinguishable.}
\label{fig:vins_rmse_overtime}
\vspace{-1.0em}
\end{figure}

\begin{table}[htbp]
\centering
\caption{The average position and orientation RMSE in Monte Carlo simulations under different measurement noise levels.}
\setlength{\abovecaptionskip}{0cm}
\setlength{\belowcaptionskip}{0cm}
\setlength{\tabcolsep}{11pt}
\renewcommand{\arraystretch}{1.5}
\begin{tabular}{c|ccc}
\hline
{Estimator}     & {1 pixel} & {3 pixel} & {5 pixel} \\ 
\hline
& \multicolumn{3}{c}{Position / Orientation RMSE (m / deg)}  \\
\hline
EKF &   0.114 / 0.407   &  0.324 / 1.528  & 0.496 / 2.355   \\
FEJ &  0.103 / 0.279   & 0.290 /  0.963  &  0.610 / 2.260   \\
I-EKF &  \textbf{0.102} / \textbf{0.274}    & \textbf{0.238} / 0.708  &   0.366 / 1.049   \\
T-EKF &  \textbf{0.102} / \textbf{0.274}    & \textbf{0.238} / \textbf{0.706}  &   \textbf{0.365} / \textbf{1.046}   \\
\hline
\end{tabular}
\label{fig:vins_armse}
\vspace{-1.0em}
\end{table}

From Table \ref{fig:vins_armse} and Figure \ref{fig:vins_rmse_overtime}, it can be observed that the three estimators: FEJ, I-EKF, and T-EKF, behave comparably and significantly surpass the classical EKF in terms of accuracy when the measurement noise is low. However, as measurement noise increases, the FEJ's performance degrades substantially, and in some cases it even falls below that of the classical EKF. This degradation occurs because the FEJ selects the first state estimate as linearization points, which violates the first-order optimality. Owning to this, poor initial estimates caused by high measurement noises might lead to non-negligible linearization errors, thereby reducing its accuracy. In contrast, both the I-EKF and T-EKF adopt the current best estimates as linearization points, ensuring both first-order optimality and correct observability properties. Consequently, these two estimators consistently deliver accurate and reliable estimation results, significantly outperforming the classical EKF across all levels of measurement noises.

\begin{figure}[htbp]
\centering
\includegraphics[width=0.99\linewidth]{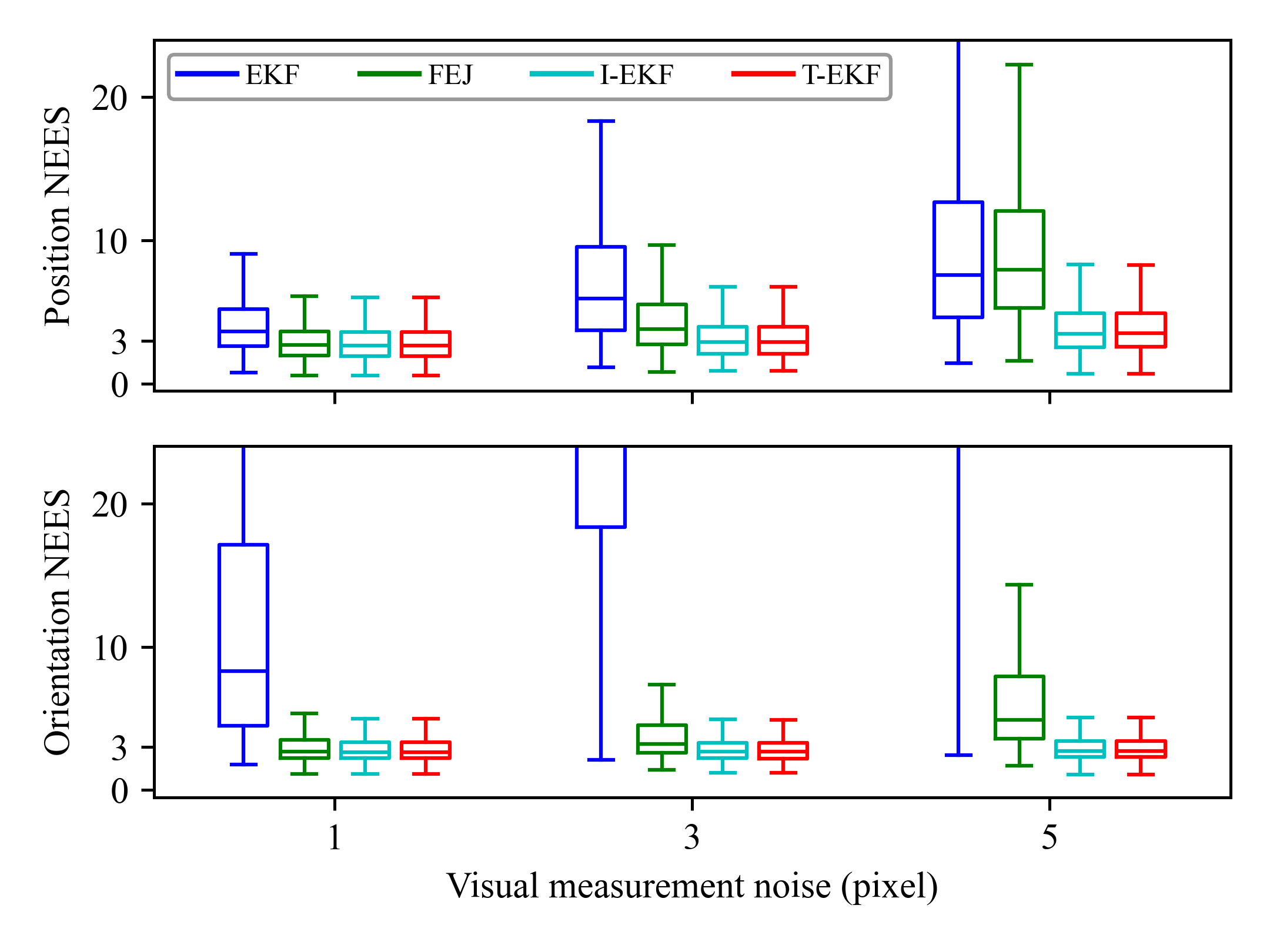}
\caption{The plots of position and orientation NEES in Monte Carlo simulations under various measurement noise levels. }
\label{fig:VINS_NEES}
\vspace{-1.0em}
\end{figure}

As seen from Figure \ref{fig:VINS_NEES}, due to mistaking the unobservable global orientation as observable, the classical EKF exhibits severe inconsistency phenomena, especially in the result of orientation estimation. Benefiting from guaranteeing correct observability properties, the FEJ exhibits improved consistency under low measurement noise levels. However, owning to suffering from linearization errors caused by poor initial estimates, its consistency significantly deteriorates as measurement noises increase. In comparison, the I-EKF and the T-EKF always remain consistent across different measurement noise levels. This demonstrates the stability and reliability of the proposed method, especially in highly challenging scenarios.

\subsubsection{Computational Efficiency Evaluation}
Next, we evaluate the computational efficiency of these estimators. It is important to note that the two proposed implementation methods, i.e., T-EKF 1 and T-EKF 2, differ in computational efficiency. As a result, we examine the two methods separately. Figure \ref{fig:VINS_timing} depicts the average running time required by these estimators to process single-frame data. 

\begin{figure}[htbp]
\centering
\includegraphics[scale = 0.56]{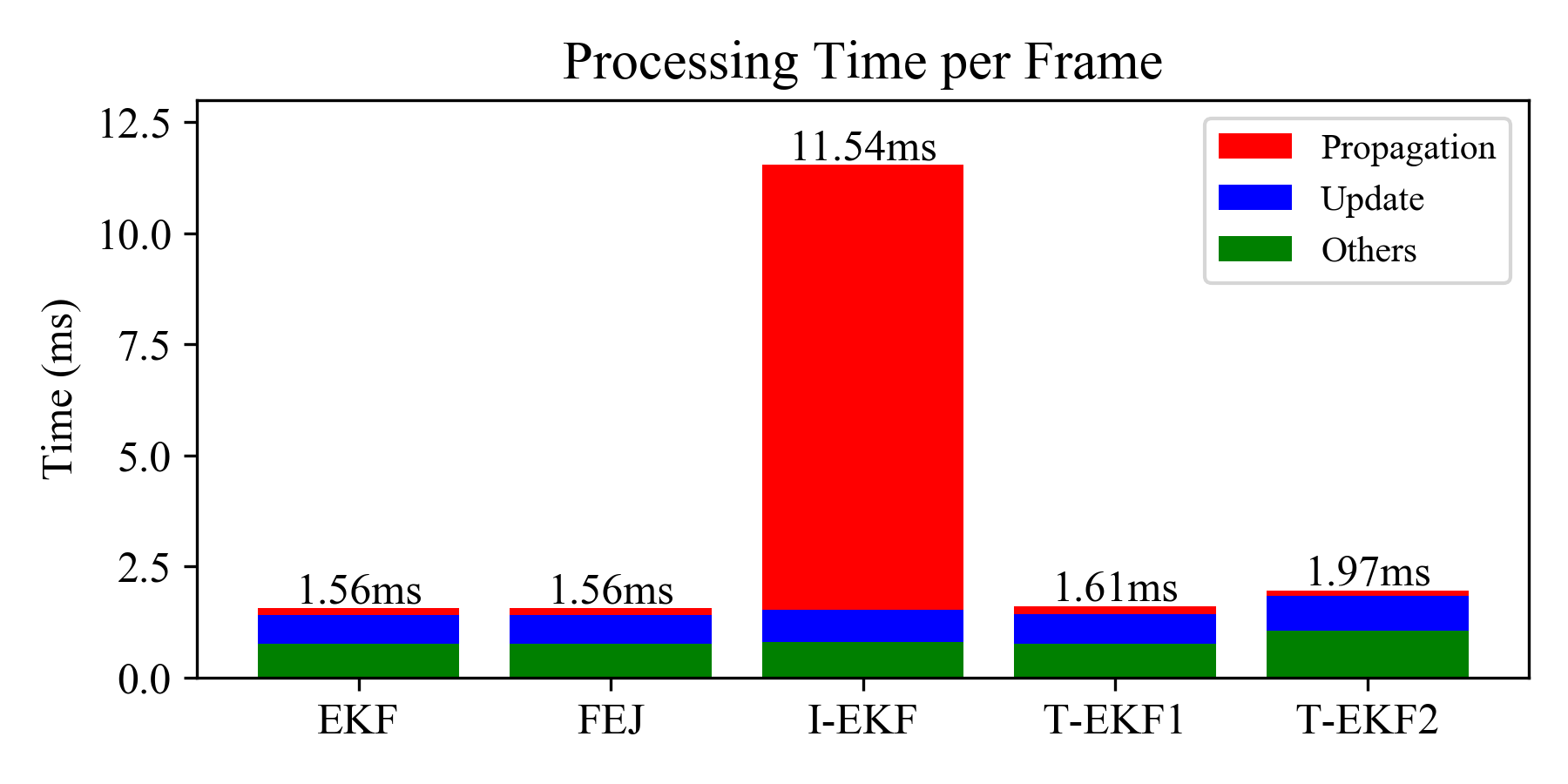}
\caption{The plots of average time required to process single-frame data. The tests are run on a single core of AMD® Ryzen 9 3900x CPU.}
\label{fig:VINS_timing}
\vspace{-1.0em}
\end{figure}

As evident, the classical EKF and the FEJ are the most efficient due to their simple realizations. In contrast, the I-EKF behaves the worst with the running time significantly exceeding that of others. This is mainly because the I-EKF suffers from the computational bottleneck associated with its quadratic complexity to the number of landmarks in covariance propagation~\cite{B32}. In comparison, the proposed two implementation methods (i.e., T-EKF 1 and T-EKF 2), achieve efficiency comparable to the classical EKF and the FEJ. This is achieved by leveraging the covariance propagation technique \cite{B70} (T-EKF 1) and implementing covariance propagation in the original coordinates (T-EKF 2), respectively, thereby circumventing the computational bottleneck of the I-EKF. Notably, compared to the T-EKF 1, the T-EKF 2 consumes more running time. This is primarily because OpenVINS requires multiple state updates with the time-consuming covariance correction operation (see Line $22$ in Algorithm \ref{alg:a2}) executed multiple times.

\subsection{Real-world Experiments}

We further validate the performance of these estimators using the publicly available EuRoC~\cite{B72} and TUM-VI~\cite{B73} datasets. To ensure a fair evaluation, we employ the default configuration parameters in OpenVINS. Table \ref{tab:VINS_euroc_rmse} presents the root mean square error (RMSE) for position and orientation estimation across the two datasets. The best results are highlighted in bold, while the second-best results are marked in blue. Additionally, Figure \ref{fig:VINS_MH_traj} illustrates the trajectories of these estimators on the EuRoC MH\_04\_difficult dataset.

\begin{figure}[htbp]
\centering
\includegraphics[width=0.45\textwidth]{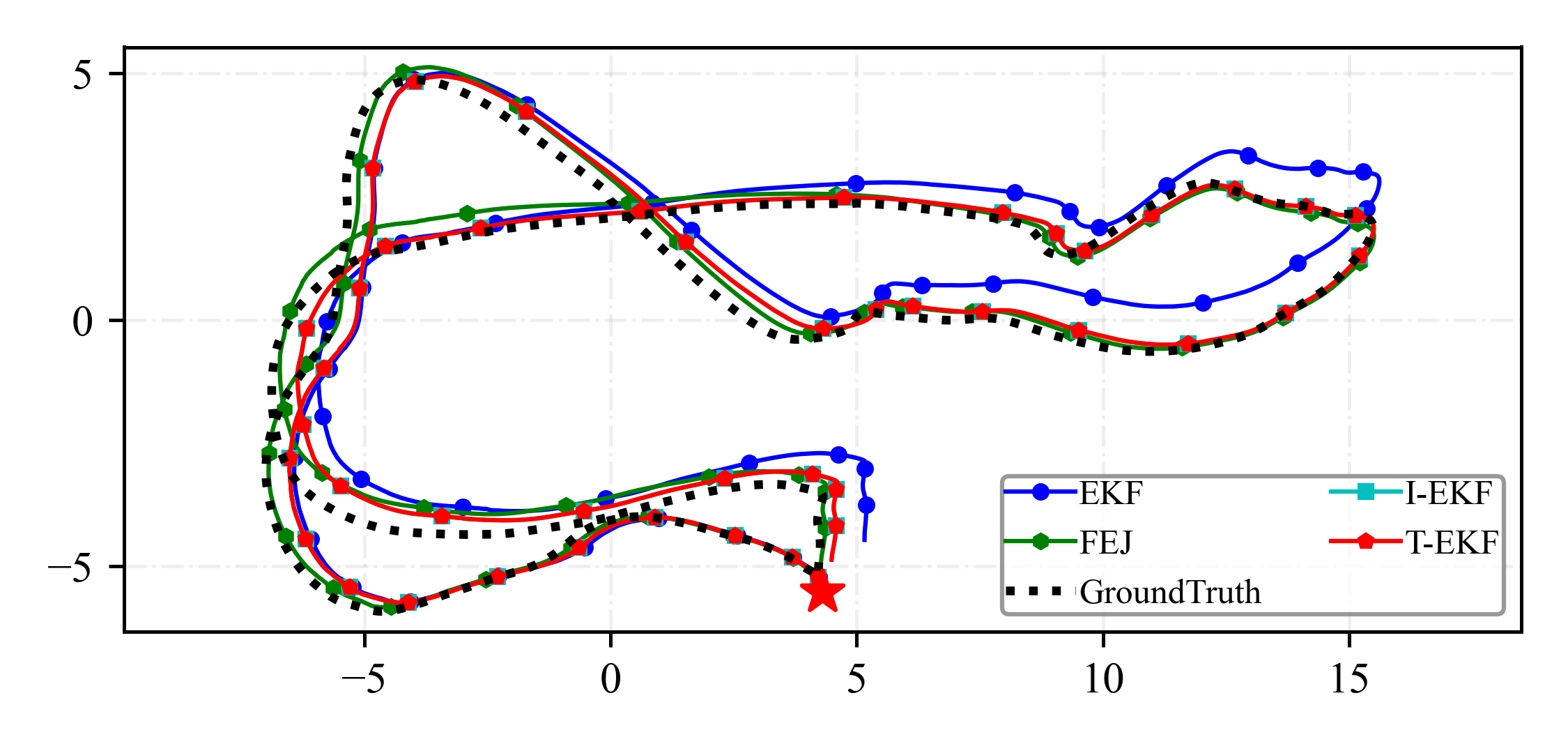}
\caption{The estimated trajectories on the EuRoC MH\_04\_difficult dataset. These trajectories are aligned at the first frame and the common initial position is represented as a red star.
Notably, the trajectories of the I-EKF and the T-EKF almost coincide and are difficult to distinguish.}
\label{fig:VINS_MH_traj}
\end{figure}

{
\def\r#1{{\color{red}#1}}
\def\b#1{{\color{blue}#1}}
\def\txtb#1{\textbf{#1}}
\def\txtr#1{\textbf{#1}}
\begin{table}
\centering
\caption{The average position and orientation RMSE on EuRoC and TUM-VI datasets.}
\label{tab:VINS_euroc_rmse}
\setlength{\abovecaptionskip}{0cm}
\setlength{\belowcaptionskip}{0cm}
\setlength{\tabcolsep}{1.5pt}
\renewcommand{\arraystretch}{1.5}
\begin{tabular}{|c|c|c|c|c|}
\hline
Dataset                 &EKF        &FEJ      &I-EKF      &T-EKF      \\ 
\hline
& \multicolumn{4}{c|}{Position / Orientation RMSE (m / deg)}  \\
\hline
V1\_01\_easy            & \b{0.084} / \txtb{1.154} & 0.117 / 1.828 & 0.094 / 1.999 & \txtb{0.080} / \b{1.640}  \\
V1\_02\_medium          & 0.115 / 1.954 & 0.118 / 1.127 & \txtb{0.102} / \b{1.037} & \b{0.110} / \txtb{1.017}  \\
V1\_03\_difficult       & 0.207 / 3.977 & 0.095 / \txtb{2.264} & \b{0.080} / 2.910 & \txtb{0.075} / \b{2.741}  \\
V2\_01\_easy            & 0.149 / 1.890 & \txtb{0.067} / \b{0.972} & 0.115 / 1.139 & \b{0.092} / \txtb{0.945}  \\
V2\_02\_medium          & 0.111 / 2.134 & 0.101 / 1.339 & \b{0.087} / \b{1.198} & \txtb{0.068} / \txtb{1.193}  \\
V2\_03\_difficult       & 0.234 / 6.428 & \b{0.162} / 2.088 & 0.166 / \b{1.746} & \txtb{0.150} / \txtb{1.416}  \\ 
Average                 & 0.205 / 3.474 & 0.110 / \b{1.602} & \b{0.107} / 1.671 & \txtb{0.095} / \txtb{1.492} \\ 
\hline
MH\_01\_easy            & 0.509 / 5.311 & \txtb{0.210} / 2.653 & 0.329 / \b{1.916} & \b{0.284} / \txtb{1.673}  \\
MH\_02\_easy            & 0.564 / 7.144 & \txtb{0.263} / \txtb{1.554} & \b{0.328} / 1.731 & 0.329 / \b{1.676}  \\
MH\_03\_medium          & \txtb{0.159} / 2.182 & {0.193} / 1.390 & 0.198 / \b{1.338} & \b{0.197} / \txtb{1.337} \\
MH\_04\_difficult       & 0.720 / \b{1.631} & \txtb{0.396} / \txtb{1.301} & 0.529 / 1.765 & \b{0.524} / 1.668 \\
MH\_05\_difficult       & 0.780 / 2.642 & 0.355 / 2.096 & \txtb{0.287} / \b{1.576} & \b{0.303} / \txtb{1.286} \\ 
Average                 & 0.546 / 3.782 & \txtb{0.283} / 1.798 & 0.334 / \b{1.665} & \b{0.327} / \txtb{1.528} \\ 
\hline
room1  & 0.101 / 3.300 & \b{0.095} / 3.024 & \txtb{0.069} / \b{1.036} & 0.117 / \txtb{0.807} \\
room2  & 0.363 / 17.245 & \b{0.176} / 1.758 & 0.185 / \b{1.150} & \txtb{0.175} / \txtb{0.804} \\
room3  & 0.261 / 9.492 & \b{0.127} / 1.448 & 0.128 / \txtb{1.156} & \txtb{0.113} / \b{1.287} \\
room4  & 0.076 / 3.492 & \txtb{0.033} / \txtb{0.595} & 0.065 / 0.653 & \b{0.057} / \b{0.640} \\
room5  & 0.186 / 6.127 & \txtb{0.118} / \b{1.026} & \b{0.120} / 1.162 & 0.143 / \txtb{0.982} \\
room6  & 0.095 / 5.796 & \txtb{0.055} / \b{3.030} & 0.062 / \txtb{2.158} & \b{0.061} / 3.252 \\
Average                 & 0.180 / 7.575 & \txtb{0.100} / 1.813 & \b{0.104} / \txtb{1.219} & 0.110 / \b{1.295} \\ 
\hline
\end{tabular}
\vspace{-1.0em}
\end{table}
}

As demonstrated, the three variants of the classical EKF, namely the FEJ, the I-EKF, and the T-EKF, achieve superior accuracy and significantly outperform the classical EKF in both position and orientation estimation. This superiority is due to the fact that these estimators guarantee correct observability properties, thereby eliminating the inconsistency caused by observability mismatch. Among these estimators, the T-EKF exhibits particularly competitive and even better estimation results, especially in orientation estimation. Notably, the visual measurement noise in these datasets is lower than the settings used in the simulation. As a result, the FEJ is nearly unaffected by the linearization errors and performs comparably to the I-EKF and the T-EKF across all datasets. 

Furthermore, to provide deeper insight into the consistency of these estimators, we plot the normalized estimation error squared (NEES) for position and orientation estimation on the EuRoC V1\_02\_medium dataset, as shown in Figure \ref{fig:VINS_V12_nees}. As observed, the NEES values of the classical EKF significantly exceed those of the other estimators, particularly in terms of orientation estimation. This suggests that the classical EKF acquires spurious information along directions that are erroneously mistaken as observable, leading to significant overconfidence in the estimation results for those directions. In contrast, the T-EKF, the FEJ, and the I-EKF exhibit more reasonable NEES values, indicating enhanced consistency.

\begin{figure}[htbp]
\centering
\includegraphics[width=0.49\textwidth]{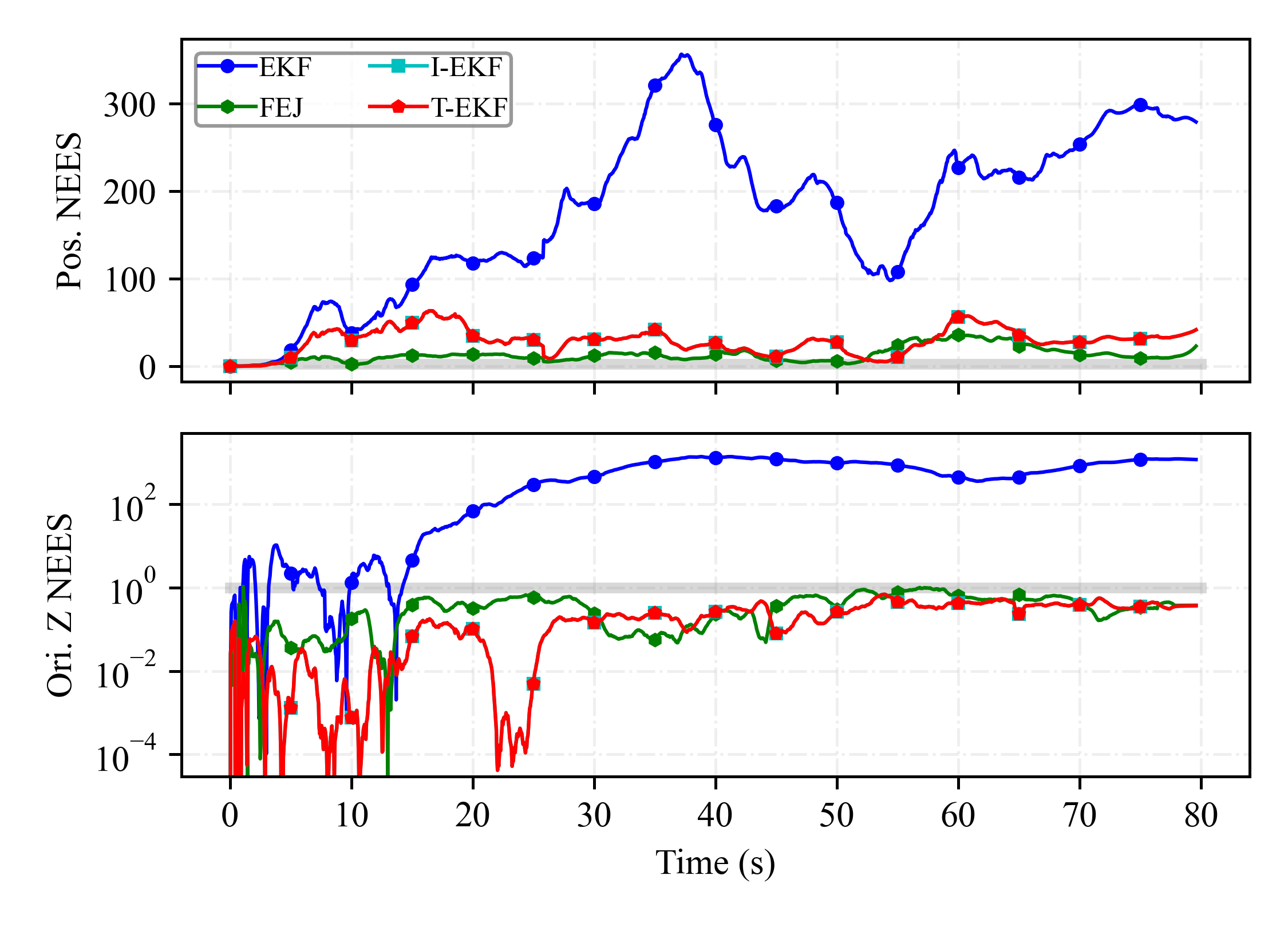}
\caption{The plots of position and orientation NEES on the EuRoC MH\_04\_difficult dataset. The results of the I-EKF and the T-EKF nearly coincide and are hard to distinguish.}
\label{fig:VINS_V12_nees}
\end{figure}

\section{Beyond the Results: A Dissuasion}
\label{sec:discussion}


In this paper, we have introduced a novel framework for addressing the inconsistency problem in the estimation of partially observable nonlinear systems, particularly within the context of the Extended Kalman Filter (EKF). The proposed framework involves the design of a linear time-varying transformation, which is always exists and can be analytically constructed, along with a transformation-based estimator. Specifically, our framework is grounded in an in-depth theoretical analysis of the observability properties of the EKF estimator system and their relationship to the underlying nonlinear system. The key insight is that the unobservable subspace of the EKF estimator system is state-independent and belongs to the unobservable subspace of the original nonlinear system. This insight has led to the development of a systematic solution, thus offering a new perspective and methodology in the field. 


\subsection{Connections}


Matrix Lie group-based methods utilize the geometric structure of the system to maintain consistency in state estimation, often leveraging the properties of matrix Lie groups to preserve intrinsic system dynamics. The Robot-centric methods focus on reformulating the state relative to a local moving frame that is attached to the robot. Our proposed solution leverages linear time-varying transformations to ensure consistent state estimation. This approach does not require complex geometric or application-specific assumptions, making it more flexible and applicable to a broader range of nonlinear estimation problems.

Although these methods may appear to operate under fundamentally different paradigms, our proposed approach establishes an underlying connection among existing state estimation approaches. Specifically, by constructing appropriate time-varying transformations that systematically reshape the observer dynamics, we demonstrate that the transformed estimator systems are equivalent to those of both robot-centric and matrix Lie group-based estimators. 
The proposed framework not only offers a unified perspective for analyzing and understanding existing estimation architectures but may also contribute to the development of advancing state estimation techniques. This integrative understanding could potentially serve as a foundation for exploring new hybrid approaches that combine the strengths of different estimation philosophies while mitigating their individual limitations.

\subsection{Advantages}

\subsubsection{Accuracy and Consistency}

The proposed framework exhibits superior accuracy and consistency performance. As simulations and experiments demonstrate, the T-EKF estimator outperforms observability constraint-based estimators (e.g., FEJ and OC) in terms of accuracy and consistency. Furthermore, it is competitive, and in some cases superior, to those of matrix Lie group-based estimators (e.g., I-EKF). This superiority arises mainly from the fact that the estimator systems used by the T-EKF and I-EKF is theoretically optimal, unlike the FEJ and OC, which do not conform to the first-order approximation. This discrepancy becomes particularly evident in scenarios involving intermittent observations (see Section \ref{sec:application_cl}) or substantial measurement noise (see Section \ref{sec:application_vins}), where the linearization errors associated with the FEJ and OC are no longer negligible.

%

\subsubsection{Computational Efficiency}

The proposed framework demonstrates superior computational efficiency compared to matrix Lie group-based methods. Moreover, it enables the simultaneous resolution of consistency and computational efficiency through the design of transformations, thereby alleviating the computational bottleneck associated with robot-centric and matrix Lie group-based estimators (see Section \ref{sec:application_vins}). Additionally, this framework may address consistency, accuracy, and other performance concerns by combining T-EKF 1 and T-EKF 2.


\subsubsection{Simplicity of Realization}
The proposed framework is straightforward to implement. Linear transformations are always existing and can be designed using analytical construction methods. In comparing to the empirical design of state representations in robot-centric and matrix Lie group-based estimators, the design of transformations follows a clear procedure and has a theoretical guarantee. More importantly, introducing linear transformations into filters, such as EKF and MSCKF \cite{???}, is simpler to implement and convenient to integrate into existing frameworks like OpenVINS \cite{???}.


\subsubsection{Versatility}

Our framework offers a flexible and versatile solution for consistent state estimation in arbitrary partially observable nonlinear systems, without requiring complex geometric or application-specific assumptions. Multiple realizations of the proposed estimators allow customization to specific application requirements, enabling users to exploit more advantages of our methodologies. 


\subsection{Limitations}

Despite the significant advantages in terms of simplicity and effectiveness provided by the introduction of linear time-varying transformations, the proposed framework still has limitations. Specifically, these transformations are currently only adapt to applications in linearization-based nonlinear estimators, such as the EKF and Maximum A Posteriori (MAP) estimation. Nonlinear estimators, including the Unscented Kalman Filter (UKF) and Particle Filter (PF), cannot be applied directly for now. It might require \emph{nonlinear transformations}. However, the design of nonlinear transformations remains still a great challenge, and lacks canonical approaches. This limitation may hinder further improvements in estimation accuracy.


\section{Conclusions And Future Work}
\label{sec:conclusion}

In this paper, we theoretically revealed the relationship of the unobservable subspace between the EKF estimator system and the actual system, and established the necessary and sufficient conditions for observability matching. Based on the theoretical finding, we proposed two design methodologies to construct linear time-varying transformations that render the unobservable subspace of the transformed system constant, thereby circumventing the observability mismatch issue. By leveraging the designed transformations, we further proposed two equivalent, consistent transformation-based EKF estimators. We conducted both simulations and experiments on several representative applications and validated the advantages of our approach in terms of consistency, computational efficiency, and practical realizations.


In our current work, we have witnessed the advantages of transformation-based methods in improving consistency and computational efficiency. However, their potential capabilities in other aspects, such as enhancing convergence and accuracy, have not been fully explored. We will discuss these topics in future work.

\appendix

\subsection{Proof of Lemma \ref{propoition:pre1}}
\label{app:propoition:pre1}

\begin{proof}
In light of \eqref{equ:invariant2}, for the nominal linearized system \eqref{equ:linear_state}-\eqref{equ:linear_measure}, the unobservable subspace $\Ker (\boldsymbol{\mathcal{O}}_{k})$ should belong to the kernel space of the linearized measurement Jacobian matrix $\mathbf{H}_{k}$, that is,
\begin{equation} \label{equ:subspace}
\Ker(\boldsymbol{\mathcal{O}}_k) \subseteq \Ker(\mathbf{H}_{k}) .
\end{equation}
Referring to \eqref{equ:trans_h}, the linearized measurement Jacobian matrix $\mathbf{H}_{k}$ is a matrix-valued function of the nominal linearization point $\mathbf{x}_{k}^{\ast}$, i.e., 
$$
\mathbf{H}_{k}=\mathbf{H}(\mathbf{x}_{k}^{\ast}) ,
$$
As a consequence, any element of the kernel space $\Ker(\mathbf{H}_{k})$ is a vector-valued function of the nominal linearization point $\mathbf{x}_{k}^{\ast}$. Thus the elements of $\Ker(\boldsymbol{\mathcal{O}}_k)$ are also vector-valued functions of the nominal linearization point $\mathbf{x}_{k}^{\ast}$, which completes the proof.
\end{proof}

\subsection{Proof of Theorem \ref{theorem_1_invariant}}
\label{app:theorem_1_invariant}

\begin{proof}
We conduct the proof by contradiction. Suppose that $\Ker(\hat{\boldsymbol{\mathcal{O}}}_k)$ is state-dependent and let $\mathbf{B}_{k}(\hat{\mathbf{x}}_{k|k-1})$ be a basis matrix of $\Ker(\hat{\boldsymbol{\mathcal{O}}}_k)$, i.e.,
$$
\Ker(\hat{\boldsymbol{\mathcal{O}}}_k)
=
\Span_{\col} \left( \mathbf{B}_{k}(\hat{\mathbf{x}}_{k|k-1}) \right) .
$$
Then, in light of \eqref{equ:invariant1}, it can be obtained that
\begin{equation} \label{equ:contradiction_con}
\hat{\mathbf{F}}_{k} \Span_{\col} \left( \mathbf{B}_{k}(\hat{\mathbf{x}}_{k|k-1}) \right) \subseteq \Span_{\col} \left( \mathbf{B}_{k+1}(\hat{\mathbf{x}}_{k+1|k}) \right) .
\end{equation}
Combing with \eqref{equ:prop_jac_2}, it can be derived that
\begin{equation} \label{equ:cont_obey_con}
\begin{aligned}
\mathbf{F}(\hat{\mathbf{x}}_{k|k}) \Span_{\col} \left( \mathbf{B}_{k}(\hat{\mathbf{x}}_{k|k-1}) \right) \subseteq \Span_{\col} \left( \mathbf{B}_{k+1}(\hat{\mathbf{x}}_{k+1|k}) \right)  
\end{aligned}
\end{equation}
Notably, $\hat{\mathbf{x}}_{k|k}$ and $\hat{\mathbf{x}}_{k|k-1}$ obey the state update equation of the classical EKF as follows 
$$
\hat{\mathbf{x}}_{k|k} = \hat{\mathbf{x}}_{k|k-1} + \mathbf{K} (\mathbf{y}_{k} - \mathbf{h}(\hat{\mathbf{x}}_{k|k-1}, \mathbf{0})) ,
$$
while $\hat{\mathbf{x}}_{k+1|k}$ and $\hat{\mathbf{x}}_{k|k}$ follow the state prediction equation of the classical EKF, i.e.,
$$
\hat{\mathbf{x}}_{k+1|k} = \mathbf{f}(\hat{\mathbf{x}}_{k|k}, \mathbf{u}_{k}, \mathbf{0}) .
$$
Substituting the two equations into \eqref{equ:cont_obey_con} yields 
\begin{equation} \label{equ:cont_obey_con_/}
\small
\begin{aligned}
&\mathbf{F}(\hat{\mathbf{x}}_{k|k-1} + \mathbf{K} (\mathbf{y}_{k} - \mathbf{h}(\hat{\mathbf{x}}_{k|k-1}, \mathbf{0}))) \Span_{\col} \left( \mathbf{B}_{k}(\hat{\mathbf{x}}_{k|k-1}) \right) \subseteq \\
&\Span_{\col} \left( \mathbf{B}_{k+1}(\mathbf{f}(\hat{\mathbf{x}}_{k|k-1} + \mathbf{K} (\mathbf{y}_{k} - \mathbf{h}(\hat{\mathbf{x}}_{k|k-1}, \mathbf{0})), \mathbf{u}_{k}, \mathbf{0})) \right)
\end{aligned}
\end{equation}
which cannot be guaranteed to hold since $\mathbf{y}_{k}$ is random and independent of the prediction $\hat{\mathbf{x}}_{k|k-1}$. This disobeys the forward $\hat{\mathbf{F}}_{k}$-invariance. Thus the unobservable subspace is independent of the state prediction value. Using the same manner, we can also prove that the unobservable subspace is independent of the state estimation value. Consequently, the unobservable subspace $\Ker(\hat{\boldsymbol{\mathcal{O}}}_k)$ of the estimator's system is state-independent. 
\end{proof}

\subsection{Proof of Theorem \ref{theorem_2_invariant}}
\label{app:theorem_2_invariant}

\begin{proof}
According to Theorem \ref{theorem_1_invariant}, it is known that the unobservable subspace of the estimator's linearized system is state-independent (constant). Thus, let $\hat{\mathbf{B}}_k$ denote a basis matrix of the unobservable subspace $\Ker (\hat{\boldsymbol{\mathcal{O}}}_{k})$ of the estimator's linearized system, i.e.,
\begin{equation} \label{threm2:kernel_1}
\Ker (\hat{\boldsymbol{\mathcal{O}}}_{k}) = \Span_{\col} \left( \hat{\mathbf{B}}_k \right) .
\end{equation}
It obeys the properties given by \eqref{equ:invariant1}-\eqref{equ:invariant2}, which can be jointly written as
\begin{equation} \label{equ:obs_hbk_0}
\hat{\mathbf{F}}_k \Span_{\col} \left( \hat{\mathbf{B}}_k \right) 
\subseteq 
\Span_{\col} \left( \hat{\mathbf{B}}_{k+1} \right) 
\subseteq  
\Ker(\hat{\mathbf{H}}_{k+1}) .
\end{equation}
Recalling \eqref{equ:trans_f}-\eqref{equ:trans_h} and \eqref{equ:prop_jac_2}-\eqref{equ:update_jac_2}, $\mathbf{F}_k$ and $\hat{\mathbf{F}}_k$ are the matrix-valued function of $\mathbf{x}_{k}^{\ast}$ and $\hat{\mathbf{x}}_{k|k}$, respectively, while $\mathbf{H}_{k+1}$ and $\hat{\mathbf{H}}_{k+1}$ are the matrix-valued function of $\mathbf{x}_{k+1}^{\ast}$ and $\hat{\mathbf{x}}_{k+1|k}$, respectively, as follows
$$
\begin{aligned}
&\hat{\mathbf{F}}_k = \mathbf{F}(\hat{\mathbf{x}}_{k|k}), \quad 
\mathbf{F}_k = \mathbf{F}(\mathbf{x}_{k}^{\ast}) , \\
&\hat{\mathbf{H}}_{k+1} = \mathbf{H}(\hat{\mathbf{x}}_{k+1|k}), \quad
\mathbf{H}_{k+1} = \mathbf{H}(\mathbf{x}_{k+1}^{\ast}) .
\end{aligned}
$$
Since both $\hat{\mathbf{B}}_{k}$ and $\hat{\mathbf{B}}_{k+1}$ are state-independent (constant), it can be derived from \eqref{equ:obs_hbk_0} that
\begin{equation} 
\mathbf{F}_k \Span_{\col} \left( \hat{\mathbf{B}}_k \right) 
\subseteq 
\Span_{\col} \left( \hat{\mathbf{B}}_{k+1} \right) 
\subseteq  
\Ker(\mathbf{H}_{k+1}) .
\end{equation}
for any $k \geq 0$. By iterating forward, it can be derived that 
$$
\mathbf{F}_{k+\ell-1} \cdots \mathbf{F}_k \Span_{\col} \left( \hat{\mathbf{B}}_k \right) 
\subseteq 
\Span_{\col} \left( \hat{\mathbf{B}}_{k+\ell} \right) 
\subseteq  
\Ker(\mathbf{H}_{k+\ell}) 
$$
which implies that
$$
\mathbf{H}_{k+\ell} \mathbf{F}_{k+\ell-1} \cdots \mathbf{F}_{k} \hat{\mathbf{B}}_k = \mathbf{0}
$$
for any $\ell \geq 0$. Therefore, we have
$$
\boldsymbol{\mathcal{O}}_{k} \hat{\mathbf{B}}_k = \mathbf{0} .
$$
That is, 
\begin{equation} \label{threm2:kernel_2}
\Span_{\col} \left( \hat{\mathbf{B}}_k \right)
\subseteq
\Ker \left( \boldsymbol{\mathcal{O}}_{k} \right) .
\end{equation}
Combing \eqref{threm2:kernel_1} and \eqref{threm2:kernel_2} yields
$$
\Ker ( \hat{\boldsymbol{\mathcal{O}}}_{k} )
\subseteq
\Ker ( \boldsymbol{\mathcal{O}}_{k} )
$$
which completes the proof.
\end{proof}

\subsection{Proof of Theorem \ref{theorem:final}}
\label{app:theorem:final}

\begin{proof}
According to Theorem \ref{theorem_1_invariant} and Theorem \ref{theorem_2_invariant}, since the unobservable subspace of the estimator's linearized system is state-independent (constant), the unobservable subspace of the two systems does not coincide if the unobservable subspace of the nominal linearized system is state-dependent (\textbf{Proof of Necessity}). To complete the proof, we only need to prove that the unobservable subspace of the nominal linearized system belongs to that of the estimator system if it is state-independent. Let $\mathbf{B}_k$ be a basis matrix of the unobservable subspace of the nominal linearized system \eqref{equ:linear_state}-\eqref{equ:linear_measure}, which is state-independent (constant). In light of \eqref{equ:invariant1}, it can be obtained that
\begin{equation} \label{equ:const_inv}
\mathbf{F}(\mathbf{x}_{k}^{\ast}) \Span_{\col} \left( \mathbf{B}_{k} \right) \subseteq \Span_{\col} \left( \mathbf{B}_{k+1} \right) .    
\end{equation}
Substituting $\hat{\mathbf{x}}_{k|k}$ into \eqref{equ:const_inv} by replacing $\mathbf{x}_{k}^{\ast}$ yields
\begin{equation} \label{equ:const_est}
\mathbf{F}(\hat{\mathbf{x}}_{k|k}) \Span_{\col} \left( \mathbf{B}_{k} \right) \subseteq \Span_{\col} \left( \mathbf{B}_{k+1} \right) 
\end{equation}
which can be equivalently written as
\begin{equation} \label{equ:const_est_iter}
\hat{\mathbf{F}}_{k} \Span_{\col} \left( \mathbf{B}_{k} \right) \subseteq \Span_{\col} \left( \mathbf{B}_{k+1} \right) .        
\end{equation}
Moreover, in light of \eqref{equ:invariant2}, we have
\begin{equation} \label{equ:const_inv_h}
\Span_{\col} \left( \mathbf{B}_{k} \right) \subseteq \Ker(\mathbf{H}(\mathbf{x}_{k}^{\ast})) .
\end{equation}
Substituting $\hat{\mathbf{x}}_{k|k-1}$ into \eqref{equ:const_inv_h} by replacing $\mathbf{x}_{k}^{\ast}$ yields
\begin{equation} \label{equ:const_est_iter_h}
\Span_{\col} \left( \mathbf{B}_{k} \right) \subseteq \Ker(\hat{\mathbf{H}}_{k})
\end{equation}
Combing \eqref{equ:const_est_iter} and \eqref{equ:const_est_iter_h} and iterating forward, it can be derived that
$$
\hat{\mathbf{F}}_{k+\ell-1} \cdots \hat{\mathbf{F}}_{k} \Span_{\col} \left( \mathbf{B}_{k} \right) \subseteq \Ker(\hat{\mathbf{H}}_{k+\ell})
$$
for any $\ell \geq 0$. Then it can be concluded that
$$
\Span_{\col} \left( \mathbf{B}_{k} \right) \subseteq \Ker (\hat{\boldsymbol{\mathcal{O}}}_{k}) .
$$
Together with Theorem \ref{theorem_2_invariant}, we have
$$
\Ker (\hat{\boldsymbol{\mathcal{O}}}_{k}) = \Ker (\boldsymbol{\mathcal{O}}_{k}) 
$$
if $\Ker (\boldsymbol{\mathcal{O}}_{k})$ is constant (\textbf{Proof of Sufficiency}). 
The proof is completed.
\end{proof}

\subsection{Proof of Theorem \ref{theorem:trans_design_2}}
\label{app:theorem:trans_design_2}

\begin{proof}
We complete the proof by contradiction. Suppose that the transformed state propagation Jacobian matrix is constant and the unobservable subspace of the transformed linearized system is state-dependent. Then the basis matrix of this unobservable subspace can be expressed as $\mathbf{B}_{k}(\mathbf{x}_{k}^{\ast})$, which explicitly relies on the state $\mathbf{x}_{k}^{\ast}$. Referring to the definition of the local observability matrix in \eqref{equ:obs_mat}, the unobservable subspace, i.e., the kernel space of the local observability matrix, should also be the kernel space of each block row matrix of the local observability matrix. However, if the transformed state propagation Jacobian matrix is constant, the kernel space of any $\ell$-th block row matrix explicitly depends on the state $\mathbf{x}_{k+\ell}^{\ast}$ rather than the state $\mathbf{x}_{k}^{\ast}$. As a consequence, the unobservable subspace of the transformed linearized system is state-independent.
\end{proof}

\bibliography{IEEEabrv}
\bibliographystyle{IEEETran}

\end{document}